\documentclass[ijoc,sglanonrev]{informs4_arXiv}
\usepackage{eqndefns-left} 
\RequirePackage{tgtermes}
\RequirePackage{newtxtext}
\RequirePackage{newtxmath}
\RequirePackage{bm}
\RequirePackage{endnotes}

\OneAndAHalfSpacedXII 

\usepackage{algorithm}
\usepackage{algpseudocode}
\usepackage{tikz}

\usepackage{natbib}
 \bibpunct[, ]{(}{)}{,}{a}{}{,}%
 \def\bibfont{\small}%
\EquationsNumberedThrough    

\TheoremsNumberedThrough     
\ECRepeatTheorems  %

\MANUSCRIPTNO{IJOC-0001-2024.00}

\usepackage{caption}
\usepackage[labelfont=sf]{subcaption}
\usepackage{endnotes}

\usepackage{url}
\usepackage{hyperref}

\begin{document}


\RUNAUTHOR{Wang, Y., et al.}

\RUNTITLE{Generalized Low-Rank Matrix Contextual Bandits with Graph Information}

\TITLE{Generalized Low-Rank Matrix Contextual Bandits with Graph Information}

\ARTICLEAUTHORS{%
\AUTHOR{Yao Wang}
\AFF{School of Management, Xi’an Jiaotong University, Xi’an, China, yao.s.wang@gmail.com}
\AUTHOR{Jiannan Li}
\AFF{School of Management, Xi’an Jiaotong University, Xi’an, China, jiannanli@stu.xjtu.edu.cn}
\AUTHOR{Yue Kang}
\AFF{Microsoft, Washington, United States, yuekang@microsoft.com}
\AUTHOR{Shanxing Gao}
\AFF{School of Management, Xi’an Jiaotong University, Xi’an, China, gaozn@mail.xjtu.edu.cn}
\AUTHOR{Zhenxin Xiao}
\AFF{School of Management, Xi’an Jiaotong University, Xi’an, China, zxxiao@mail.xjtu.edu.cn}
} 
\ABSTRACT{The matrix contextual bandit (CB), as an extension of the well-known multi-armed bandit, is a powerful framework that has been widely applied in sequential decision-making scenarios involving low-rank structure. 
In many real-world scenarios, such as online advertising and recommender systems, additional graph information often exists beyond the low-rank structure, that is, the similar relationships among users/items can be naturally captured through the connectivity among nodes in the corresponding graphs. However, existing matrix CB methods fail to explore such graph information, and thereby making them difficult to generate effective decision-making policies. To fill in this void, we propose in this paper a novel matrix CB algorithmic framework that builds upon the classical upper confidence bound (UCB) framework. This new framework can effectively integrate both the low-rank structure and graph information in a unified manner. Specifically, it involves first solving a joint nuclear norm and matrix Laplacian regularization problem, followed by the implementation of a graph-based generalized linear version of the UCB algorithm. Rigorous theoretical analysis demonstrates that our procedure outperforms several popular alternatives in terms of cumulative regret bound, owing to the effective utilization of graph information. A series of synthetic and real-world data experiments are conducted to further illustrate the merits of our procedure.

}


\maketitle
\section{Introduction}
The multi-armed bandit (MAB) has proven to be an effective framework for sequential decision-making, where the decision-maker progressively selects actions (i.e., arms) based on the historical data and the current state of the environment to maximize cumulative revenue (i.e., reward). However, traditional MAB approaches are limited in their ability to handle complex real-world tasks, such as user privacy protection~\citep{han2020differentially} or personalized decision-making~\citep{zhou2023spoiled}. This paper focuses on the latter by incorporating contextual information (e.g., user and item features), a setting known as the contextual bandit (CB)~\citep{li2010contextual,chu2011contextual}. Due to more informed and adaptive decisions, CB methods have been widely applied across various domains, including recommender systems~\citep{bastani2022learning, aramayo2023multiarmed}, precision medicine~\citep{bastani2020online,zhou2023spoiled}, among others~\citep{chen2020dynamic,agrawal2023tractable}.

\begin{figure}[t]
     \FIGURE
    {\includegraphics[scale=.33]{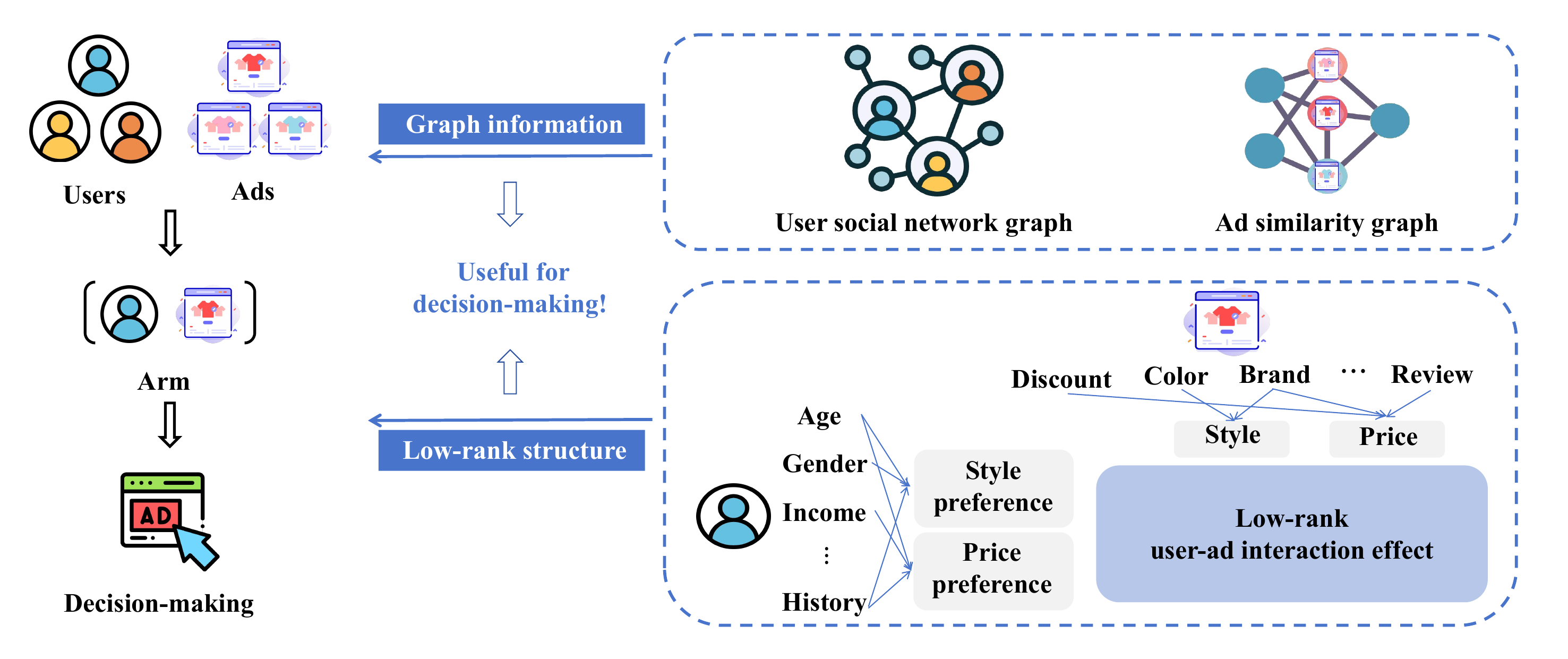}} 
{Illustration of Online Advertising in Social Networks. \label{fig:figure1}}
{}
\end{figure}

However, in some complex real-world applications, there are often significant interactions between user and item features that such vector representation based CBs are unable to effectively capture, leading to reduced decision accuracy. Taking the problem of online advertising in social networks as an example (as shown in Figure \ref{fig:figure1}), the platform, as the decision-maker, needs to progressively adjust its advertisement (ad) delivery strategies based on historical data and environmental feedback to maximize long-term revenue, making this inherently a sequential decision-making problem. In this problem, users exhibit individual differences in features namely age, gender, income, and historical behavior, resulting in varied responses to the same ad features. For example, high-income users tend to prefer well-known brands, whereas low-income users may be more price-sensitive and less brand-dependent. This indicates that there are intricate interactions between ad features (e.g., discount, color, brand, and review) and the aforementioned user features. As such, traditional CBs simply stack the user and ad features into a joint feature vector cannot sufficiently model the underlying multi-dimensional relationships. Instead, adopting a matrix-based modeling approach provides a more effective alternative. Precisely, the rows and columns of the matrix correspond to users and ads, respectively, allowing for a more systematic representation of their interactions. Furthermore, the influence of these interactions on click behavior can be modeled through a small set of low-dimensional latent factors (e.g., style and price) significantly lower than the dimension of the feature matrix. This indicates that this problem exhibits a low-rank structure, and should be handled by some designed matrix type of CB algorithms. This has also been explored in several other related problems, e.g., ~\cite{cai2023doubly,lee2024low}.

In addition to the contextual information that describes the features of an action, some certain side information is often available, capturing the similarities among different actions. This kind of relationship is typically represented as an undirected graph~\citep{lejeune2020thresholding,thaker2022maximizing}, where the nodes correspond to the actions, and the edges indicate their correlations. Consequently, this side information is commonly referred to as graph information. Continuing with the aforementioned online advertising example (as shown in the upper right of Figure \ref{fig:figure1}), advertisers can take advantage of user interaction data on social media to build a social network~\citep{zhou2020learning} while also establishing a similarity graph between ads based on their click frequencies~\citep{han2023cost}. These graph structures reveal the impact of correlations on revenue: (1) Social connections between users often indicate similar interests, meaning that friends are more likely to be interested in the same type of ads; (2) Similarities between ads often indicate substitutability, meaning that ads for products of a similar style or price tend to attract the same user. Leveraging these insights, one can recommend similar ads to users who have a social connection with the user who clicked the ad, increasing the likelihood of purchase, and thus boosting revenue of decision-maker.

Basically, both low-rank structure and graph information are two important side information for providing effective decision-making procedures in online advertising problem and beyond, e.g., \cite{lee2024low}, \cite{abdallah2025leveraging}. However, current popular bandit algorithms typically focus on exploiting just one aspect, that is, either the low-rank structure~\citep{lu2021low,kang2022efficient,jang2024efficient} or the graph information~\citep{cesa2013gang,kocak2020spectral,yang2020laplacian}. Therefore, how to effectively leverage both types of information in a unified bandit framework and further improve the decision-making performance is the main focus of this study. To this end, we shall devise a new matrix CB algorithm. Specifically, we ingeniously integrate a matrix Laplacian regularization with the low-rank matrix structure into a unified way. Furthermore, it is worth noting that the rewards of CBs are not necessarily continuous variables in some practical applications, such as the binary rewards in recommender systems and the count based rewards in ad searches. This highlights the need for more flexible generalized linear models in this new CB algorithm. Our main contributions can be summarized as follows:

\begin{enumerate}

    \item Methodologically, we introduce a novel CB algorithmic framework that effectively leverages both the low-rank structure and graph information inherent in complex decision-making scenarios. By solving an optimization problem that incorporates both the joint nuclear norm and a designed matrix Laplacian regularizer, the proposed algorithm effectively integrates both types of information while maintaining high flexibility, unifying the reward relationships under different types of non-linearity, and thus being widely applicable to a range of real decision-making problems.
    

    \item Theoretically, we use the typical cumulative regret to measure the gap between the decisions made by the proposed algorithm and the optimal decision. Through representing the sum of the aforementioned two regularization terms as an equivalent atomic norm, we can direct leverage the existing framework to derive a tighter cumulative regret bound, which includes a graph-related factor in the numerator to reflect the richness of graph information. That is, the factor decreases as the amount of graph information increases, leading to a tighter bound and highlighting the importance of leveraging such graph information. Additionally, this bound is independent of the number of actions, thus avoiding performance degradation in large action spaces.

    \item Experimentally, we demonstrate the effectiveness of our proposed algorithm by comparing the numerical results with those of existing related algorithms. We evaluate these algorithms in various synthetic scenarios including the amount of graph information, different ranks, the number of actions and further in several real-world decision-making scenarios across cancer treatment, movie recommendation and ad searches. The extensive results show that the proposed algorithm achieves more accurate decision-making policies, as evidenced by a reduction in cumulative regret and an improvement in the optimal action hit rate.
    

\end{enumerate}

\subsection{Related Literature}
\subsubsection{Low-Rank Matrix Bandits} A natural and intuitive method to develop matrix bandits is to directly generalize the multi-armed bandits from the vector representation to the matrix representation~\citep{kveton2017stochastic, katariya2017stochastic, trinh2020solving}. However, this type of method does not exploit feature information from the environment, making it difficult to achieve good decision-making performance in some complex scenarios. To address this issue, \cite{jun2019bilinear} proposed bilinear bandits, which can be viewed as a contextual low-rank matrix bandit. Building upon this work, \cite{jang2021improved} followed the bilinear setting and further leveraged the geometry of the action space to achieve an improved regret bound in terms of the matrix rank. As an extension of the aforementioned works, \cite{lu2021low} focused on a broader class of low-rank bandits, that is, generalized linear bandit models. However, its algorithm design relies on calculating covering numbers of low-rank matrices, which often suffer from computational intractability. To tackle this limitation, \cite{kang2022efficient} proposed a novel optimization objective for parameter estimation, leading to a better regret bound in terms of feature dimension. To handle heavy-tailed noise in practice, \cite{kang2024low} developed a new estimator to improve model robustness and performance. However, all the aforementioned algorithms assume that an effective exploration distribution is given, which often cannot be satisfied in practice. To this end, \cite{jang2024efficient} introduced an experimental design approach to guide the selection of the exploration distribution and derived a regret bound related to the action set.

\subsubsection{Bandits with Graph Information} The first category of approaches for integrating graph information into the bandit framework comprises clustering based methods. Specifically, a series of works~\citep{gentile2014online, korda2016distributed,yang2018graph, wang2023online}  segment users into distinct clusters and subsequently apply corresponding bandit algorithms to each user group. This process could effectively reduce the number of users that need to be individually managed, thereby enhancing overall performance. Besides, another line of works \citep{li2016collaborative,gentile2017context} considered the possibility that both users and items exhibit clustering structures, and proposed a double-clustering framework accordingly. However, such methods often overlook intra-cluster differences, which may lead to the loss of important features associated with key users or items. As such, the graph Laplacian regularization based methods~\citep{cesa2013gang,kocak2020spectral,yang2020laplacian} were proposed. By integrating some designed graph Laplacian regularizers into the optimization problem for parameter estimation, these methods enhance the accuracy of reward estimation within the bandit framework, thereby facilitating more precise decision-making. In this line of works, they primarily focus on user similarity, assuming that each user is associated with an individual parameter vector. Accordingly, their Laplacian regularization terms are designed to quantify the global impact of user similarity on the parameter matrix, where the $i$-th row represents the parameter vector of the $i$-th user. In contrast, \cite{kocak2020spectral} emphasizes item similarity under the assumption that all items share a common parameter vector, and the corresponding Laplacian regularization is used to quantify the effect of item similarity on rewards. With the advancements of deep learning techniques, some recent works~\citep{qi2022neural,qi2023graph} introduce Graph Neural Networks into the reward estimation process of the bandit framework by propagating and aggregating information over similarity graphs, thereby enhancing overall decision-making performance. These works still focus on single-aspect graph information, i.e., item similarity graph~\citep{qi2022neural} or user similarity graph~\citep{qi2023graph}. However, in real-world applications, these two aspects (i.e., dual-graph information) can be captured simultaneously. Therefore, the aforementioned single-perspective modeling approach leads to insufficient exploitation of graph information.

Although existing studies have provided important inspiration for this work, there remain significant differences. This work mainly focuses on more complex real-world scenarios where both low-rank structure and graph information are available simultaneously, whereas the aforementioned related works typically consider only one aspect---either emphasizing low-rankness or focusing on graph information. Moreover, this work further considers the coexistence of user and item similarity, which necessitates the design of new Laplacian regularization term to effectively model such dual-graph information.

\subsection{Notations and Preliminaries}
We first summarize the notations used throughout this paper. The symbols $a$, $\boldsymbol{a}$, $A$, and $\mathcal{A}$ represent scalars, vectors, matrices, and sets, respectively. For a vector $\boldsymbol{a}$, $\|\boldsymbol{a} \|$ represents the $l_2$-norm and $\|\boldsymbol{a}\|_M=\sqrt{\boldsymbol{a}^\top M \boldsymbol{a}}$ denotes the weighted norm. For a matrix $A$, $\|A\|_F$ represents the Frobenius norm, $\|A\|$ denotes the operator norm, and $\|A\|_*$ is the nuclear norm. For the operator symbols, $\odot$ represents the outer product of the vectors, $\otimes$ denotes the Kronecker product, and $\langle\cdot\rangle$ is the inner product of the vectors or matrices. $\mathcal{O}(\cdot)$ represents ignoring irrelevant constants and $\tilde{\mathcal{O}}(\cdot)$ denotes hiding logarithmic factors.

In the following, we shall introduce three key definitions.

\begin{definition}[Truncated function of rectangular matrix~\citep{minsker2018sub}] For the real-value function $\psi(\cdot)$: $\mathbb{R} \rightarrow \mathbb{R}$ is defined as  $$\psi \left(x\right)=\left\{
             \begin{array}{lc}
             \log(1+x+\frac{x^2}{2}) &  \quad x \geq 0,\\
             -\log(1-x+\frac{x^2}{2}) & \quad x <0 ,
             \end{array}
\right.$$and rectangular matrix $A \in \mathbb{R}^{d_1 \times d_2}$, let $d=d_1+d_2$, define
$$\psi_\nu(A)= \frac{[H  \mathrm{diag}\left(\psi(D_{11}), \dots , \psi(D_{dd})\right)H^\top]_{1:d_1,(d_1+1):d}}{\nu} ,$$
where $\nu > 0$, $D_{ii}$ represents the $i$-th diagonal element of the matrix $D$. Matrices $H,D$ are obtained from the full singular value decomposition (SVD) of $\widetilde{A} = H D H^\top$, where $\widetilde{A} \triangleq \left( \begin{array}{cc}
    0 & \nu A \\
     \nu A^\top & 0
\end{array}  \right)$. 
\label{def3}
\end{definition}

\begin{definition}[Score function~\citep{kang2022efficient}] 
The score function is defined as $
S(x)=-\nabla_x \log (p(x))=-\nabla_x p(x) / p(x)$, where $p(x)$ is the probability density, $\nabla_x p(x)$ is the gradient of $p(x)$ with respect to the variable $x$. For the matrix $X$, the score function is defined as $
S(X)=\left(S(X_{ij}) \right)$.
\label{s-function-def}
\end{definition}

\begin{definition}[Weighted atomic norm~\citep{rao2015collaborative}] 
For matrix $Z \in \mathbb{R}^{d_1 \times d_2}$, the weighted atomic norm with respect to the atomic set $\mathcal{A}\triangleq \{\boldsymbol{\omega}_i \boldsymbol{h}_i^\top: \boldsymbol{\omega}_i=A \boldsymbol{u}_i, \boldsymbol{h}_i=B \boldsymbol{v}_i, \|\boldsymbol{u}_i \|=\|\boldsymbol{v}_i \|=1 \}$ is defined as
$\|Z\|_{\mathcal{A}}=\inf \sum_{i}\left|c_i\right| \; s.t. \; Z=\sum_{i} c_i A_i, \; A_i \in \mathcal{A}$, where $A=P S_p^{-1 / 2}, B=Q S_q^{-1 / 2}$, $P, Q$ are the basis matrices spanning the matrix space, $S_p, S_q$ are diagonal matrices whose diagonal elements correspond to the respective weights.

\label{def5}
\end{definition}

\subsection{Organization of This Paper}

The remainder of the paper is organized as follows. Section \ref{sec2} provides a detailed formulation the studied problem. Section \ref{sec3} presents the proposed algorithm, discussing its design principles and implementation details. Section \ref{sec4} analyzes the theoretical results of the proposed algorithm, outlining key assumptions and regret bound. Section \ref{sec5} evaluates the effectiveness and superiority of our algorithm through conducting a series of experiments. Finally, Section \ref{sec6} summarizes the paper, highlights the research findings, and presents potential directions for future work.


%

\section{Problem Formulation}\label{sec2}
 



In this section,  we shall formally formulate the studied problem. Precisely, at each time step $t$, the decision-maker needs to select an item to display to the current user. Accordingly, the chosen action is represented as a specific (user, item) pair. Meanwhile, the decision-maker can observe the features of both the user and the item, denoted $\boldsymbol{p}_t \in \mathbb{R}^{d_1}$ and $\boldsymbol{q}_t \in \mathbb{R}^{d_2}$, respectively. The corresponding action can also be represented as a feature matrix $X_t = \boldsymbol{p}_t \odot \boldsymbol{q}_t \in \mathbb{R}^{d_1 \times d_2}$. 

In addition to the features described above, real-world applications often involve another side information in the form of a graph that captures the similarity among different actions. This graph information can be either explicitly provided, such as user similarity constructed from a known social network, or implicitly inferred, for example, by applying methods such as $k$-nearest neighbor, $\epsilon$-nearest neighbor to construct item similarity. We model this similarity information as an undirected graph $\mathcal{G} = (\mathcal{V}, \mathcal{E})$, where the node set $\mathcal{V} = \{1, 2, \cdots, n\}$ represents all possible actions, and the edge set $\mathcal{E}$ encodes pairwise similar relationships.

Moreover, the decision-maker interacts with the environment, subsequently obtaining the so-called reward, $y_t$. Considering the diverse nature of reward variables--such as binary variables (e.g., whether a movie is watched) in movie recommendation, continuous variables (e.g., treatment effectiveness) in precision medicine, and count variables (e.g., click-through rates) in ad searches--we can naturally assume that the relationship between the selected action and the reward belongs to the canonical exponential family defined as
$$p\left(y_t \mid X_t,\Theta^* \right)=\exp \left(\frac{y_t \langle X_t,\Theta^* \rangle-b\left(\langle X_t,\Theta^* \rangle \right)}{\phi}+c\left(y_t, \phi\right)\right),$$
$$\mathbb{E}\left(y_t \mid X_t, \Theta^* \right)=b^{\prime}\left(\left\langle X_t, \Theta^*\right\rangle\right) \triangleq \mu\left(\left\langle X_t, \Theta^*\right\rangle\right),$$
where $\Theta^* \in \mathbb{R}^{d_1 \times d_2}$ is an unknown parameter, $\phi$ is the dispersion parameter, $b(\cdot)$ is a known and strictly convex log-partition function, the first derivative of $b(\cdot)$, i.e., $\mu(\cdot)$, is called the inverse link function. The above equation can be
rewritten in the form of a generalized linear model as $y_t=\mu \left(\langle X_t,\Theta^* \rangle \right)+\epsilon_t$,
where $\epsilon_t$ is the independent sub-Gaussian noise with parameter $\omega$.

We begin by analyzing the impact of feature information on the reward, which is reflected in the low-rank structure of the parameter matrix $\Theta^*$. As mentioned previously, the reward can be modeled using a small number of low-dimensional latent factors, which are determined by the interactions between users and items. These key latent factors represent the dominant directions of interaction and are captured by the eigenspace of $\Theta^*$. In other words, the informative part of $\Theta^*$ is concentrated in a low-dimensional subspace that is much smaller than the original space (i.e., $\min\{d_1, d_2\}$), that is, $\Theta^*$ exhibits a typical low-rank structure with rank $r \ll \min\{d_1, d_2\}$.

This low-rank structure has led to a line of research on low-rank matrix contextual bandits~\citep{kang2022efficient,jang2024efficient}, where decisions are made solely based on current feature information and historical observations. However, in real-world scenarios such as social networks, each user can be represented not only by a feature vector (e.g., demographic attributes) but also by a graph that encodes user connections (e.g., friendships or interactions). Similar relational structures can also be found among items. To this end, our setting additionally leverages such relational information by introducing an undirected graph $\mathcal{G}$ over actions, which induces smoothness in the expected reward function. Specifically, for any two actions $i,j \in \mathcal{V}$, the difference in their expected rewards can be bounded by a graph-induced distance measure $\mathrm{d}(\cdot,\cdot)$ about these actions, that is,  $$\left|\mu \left( \langle X_i,\Theta^* \rangle \right)-\mu \left( \langle X_j,\Theta^* \rangle \right) \right| \leq  \mathrm{d}\left(  i , j \right) , $$ where $\mathrm{d}(\cdot,\cdot)$ measures the similarity between actions. In essence, actions that are close in the graph tend to yield similar expected rewards. This enables more accurate decisions by leveraging information propagation across action connections and uncovering relationships beyond what feature matrices alone can capture. Importantly, our use of the graph differs fundamentally from prior works based on graph feedback~\citep{gou2023stochastic,gong2025efficient}, which assume that selecting an action reveals the rewards of its neighbors on the graph. In contrast, our approach does not require such side observations, which are often costly or impractical to obtain in real-world applications.

Based on the above analysis, then the main objective is to maximize the cumulative reward \(\sum_{t=1}^T \mu \left(\langle X_t, \Theta^* \rangle\right)\). If \(\Theta^*\) is known, one can easy to choose the optimal action as \(X^* \triangleq  \arg\max_{X \in \mathcal{X}} \mu \left(\langle X, \Theta^* \rangle\right)\), where $\mathcal{X}$ is the action set. However, \(\Theta^*\) is usually unknown in practice. Therefore, we reformulate the optimization goal to minimize the following cumulative regret 
\begin{equation}R_T = \sum_{t=1}^T \left[\mu \left(\langle X^*, \Theta^* \rangle\right) - \mu \left(\langle X_t, \Theta^* \rangle\right)\right]\label{c-regret},
\end{equation}
which measures the gap
between the actions $\{X_t\}_{t=1}^{T}$ selected by the proposed algorithm and the optimal action $X^*$.

\section{Method}\label{sec3}
In this section, we shall develop an innovative algorithm named as ``Graph-Generalized Explore Subspace Then Transform'' (GG-ESTT) to minimize the cumulative regret defined in (\ref{c-regret}). Then the core challenge lies in the fact that the true parameter matrix $\Theta^*$ is unknown, necessitating the estimation involved in the decision-making process. Given that the unknown $\Theta^*$ exhibits a low-rank structure, we focus on estimating the column and row subspaces to fully exploit this structural characteristic. In addition, we shall further leverage the available graph information to improve the accuracy of such subspaces estimation. We then utilize the revealed space redundancy from this stage to transform the original space into an almost low-dimensional parameter space. Subsequently, in this transformed space, we continue to utilize graph information to design an efficient action selection strategy. The detailed procedure is summarized in Algorithm \ref{algo1}.

\begin{algorithm}[t]
\caption{Graph-Generalized Explore Subspace Then Transform (GG-ESTT)} 
\hspace*{0.02in} {\bf Input:} 
$\mathcal{X}, \lambda, \alpha, T_1, T, r, \tau, \lambda_2, \lambda_{\perp}, L$. 
\begin{algorithmic}[1]
\For{$t=1$ to $T_1$} \label{algo1_line1}
\State Choose action $X_t \in \mathcal{X}$ according to distribution $\mathbb{D}$, and receive reward $y_t$. \label{algo1_line2}
\EndFor \label{algo1_line3}
\State Compute the estimator $\hat{\Theta}_{T_1}$ by Equation \eqref{eq2}.\label{algo1_line4}
\State Rotate the action set and parameter space as followed:
\begin{equation*}
    \mathcal{X}^\prime = \left\{\left(\hat{U}, \hat{U}_{\perp}\right)^{\top} X \left(\hat{V}, \hat{V}_{\perp}\right): X\in \mathcal{X}\right\},
    \Theta^\prime=\left(\hat{U}, \hat{U}_{\perp}\right)^{\top} \Theta^* \left(\hat{V}, \hat{V}_{\perp}\right),
\end{equation*}
where $\hat{U}\in \mathbb{R}^{d_1 \times r}, \hat{U}_{\perp}\in \mathbb{R}^{d_1 \times d_1-r}, \hat{V}\in \mathbb{R}^{d_2 \times r}, \hat{V}_{\perp}\in \mathbb{R}^{d_2 \times d_2-r}$ are obtained from the full SVD of $\hat{\Theta}_{T_1} =(\hat{U}, \hat{U}_{\perp}) S (\hat{V}, \hat{V}_{\perp})^\top$. \label{algo1_line5}
\State Obtain the vectorized parameter $\theta^{*}$ and its corresponding action set $\mathcal{X}_{\text {vec }}^{\prime}$ according to \eqref{eq4} and \eqref{eq6},  respectively, such that the last $(d_1-r) \cdot (d_2-r)$ components lie in the redundant space. \label{algo1_line6}
\For{$t=1$ to $T_2=T-T_1$} \label{algo1_line7}
\State Run Graph-LowGLM-UCB (Algorithm
\ref{algo2}) with action set $\mathcal{X}_{\text {vec }}^{\prime}$ and the low dimension $k=(d_1+d_2-r)r$. \label{algo1_line8}
\EndFor \label{algo1_line9}

\end{algorithmic}
\label{algo1}
\end{algorithm}

\subsection{Estimation of Subspace with Graph Information}\label{sec3.1}



The first stage (lines 1-4) of the proposed Algorithm \ref{algo1} is to leverage graph information for improving the estimation of the low-rank subspace. To this end, we use the observations up to time $T_1$ to obtain an initial estimate of the unknown parameter $\Theta^*$ by solving an optimization problem with two regularization terms capturing the low-rank structure and graph information, respectively. Specifically, due to the non-linearity of the reward function, we adopt a commonly-used quadratic loss function~\citep{plan2016generalized, yang2017high}, defined as $$ L_{T_1}(\Theta)\triangleq \langle\Theta, \Theta\rangle-\frac{2}{T_1} \sum_{i=1}^{T_1}\left\langle \psi_\nu\left(y_i \cdot S\left(X_i\right)\right), \Theta\right\rangle,$$ 
where $\psi_\nu(\cdot)$ and $S(\cdot)$ are given in Definitions \ref{def3} and \ref{s-function-def}, respectively. We then use the commonly used nuclear norm to characterize the low-rank structure, defined as $$ R_n(\Theta) \triangleq \|\Theta\|_*=\inf_{U,V}\left\{\|U\|_F+\|V\|_F: \Theta=UV^\top \right\}.$$ In contrast, existing graph Laplacian regularizations only reflect the single-aspect graph information based on either user similarity or item similarity, which cannot be used to  characterize the dual-graph information considered in this work. 




We shall further show how to characterize such dual-graph information in our algorithm. For the action set $\mathcal{G}$, the similarity between actions can be captured by a symmetric matrix $W$ (i.e., the adjacency matrix), where $W_{ij} = 1$ indicates that action $i$ and action $j$ are similar (i.e., connected) and $W_{ij} = 0$ otherwise. Then, the graph Laplacian matrix of $\mathcal{G}$ is defined as $L \triangleq D-W$, where $D$ is the diagonal matrix with entries $d_{i i} \triangleq \sum_j W_{i j}$ (node degrees). With these notions, we propose a graph Laplacian regularization, which measures the smooth variation of the expected rewards over the entire action set, that is, 
\begin{equation}\label{eq1}
   R_L(\Theta^*) \triangleq \frac{1}{2} \sum_{ij} W_{ij} \left[\mu \left( \langle X_i,\Theta^* \rangle \right)-\mu \left( \langle X_j,\Theta^* \rangle \right) \right]^2
      =a_\mu \mathrm{tr}\left\{ \Theta^{*\top} \widetilde{X}^\top L \widetilde{X}  \Theta^* \right\} ,
\end{equation}
where $\widetilde{X} \triangleq \left(\mathrm{vec}(X_1),\mathrm{vec}(X_2),\dots,\mathrm{vec}(X_n)\right)^\top$. The constant $a_\mu$ is constrained by the bound of the first derivative of $\mu$. Noting that each action feature $X_i$ is a user-item interaction matrix, where rows and columns correspond to user and item features, respectively. Thus, $X_i$ simultaneously encodes information from both the user and the item. By stacking all these action features, the resulting matrix $\widetilde{X}$ serves as a joint representation of the entire action set, comprising all (user, item) pairs. Combining with the graph Laplacian matrix $L$ to encode the similarity structure among actions, this regularization term can effectively model the inherent smoothness over the dual-graph structure.




Based on the above formulation, we can obtain the estimated parameter $\hat{\Theta}_{T_1}$ by 
\begin{equation}\label{eq2}
    \widehat{\Theta}_{T_1}=\arg \min _{\Theta \in \mathbb{R}^{d_1 \times d_2}} L_{T_1}(\Theta)+\lambda R_n(\Theta)+\alpha R_L(\Theta),
\end{equation}
where \(\lambda\) and \(\alpha\) are the tuning parameters to control the influence of their corresponding regularization terms. We then perform a full SVD on the estimated $\hat{\Theta}_{T_1}$, that is, $ 
\hat{\Theta}_{T_1} = (\hat{U}, \hat{U}_{\perp}) \hat{S} (\hat{V}, \hat{V}_{\perp})^\top,
$ where $\hat{U} \in \mathbb{R}^{d_1 \times r}$, $\hat{U}_{\perp} \in \mathbb{R}^{d_1 \times (d_1 - r)}$, $\hat{V} \in \mathbb{R}^{d_2 \times r}$, and $\hat{V}_{\perp} \in \mathbb{R}^{d_2 \times (d_2 - r)}$. It is easy to see that $\hat{U}$ and $\hat{V}$ are the column and row subspaces of $\hat{\Theta}_{T_1}$, which consist its top $r$ left and right singular vectors, respectively. These subspaces provide a natural estimate of the corresponding subspaces of the true parameter matrix $\Theta^*$.


Considering the space redundancy revealed by the above subspace estimation, we further transform the original parameter space into an almost low-dimensional space (i.e., the lines 5-6 of Algorithm \ref{algo1}) to reduce the cumulative regret. More precisely, this can be done in two steps, i.e., orthogonal rotation and vectorized rearrangement. Specifically, orthogonal rotation refers to projecting the true parameter $\Theta^*$ onto the orthogonal space spanned by all singular vectors of the estimated $\hat{\Theta}_{T_1}$. This gives a rotated version of $\Theta^*$, defined as $ \Theta^{\prime} = \left( \hat{U}, \hat{U}_{\perp} \right)^\top \Theta^* \left( \hat{V}, \hat{V}_{\perp} \right)$. To ensure the equivalence of the problem, it is necessary to reformulate the expected reward as
\begin{equation}\label{rotated-eq}
    \mu \left(\langle X, \Theta^* \rangle \right)=\mu \left( \mathrm{tr} \left\{ \left[ \left( \hat{U}, \hat{U}_\perp \right)^\top X \left( \hat{V}, \hat{V}_\perp \right) \right]^\top \Theta' \right\} \right)=\mu \left(  \left\langle X^{\prime} ,\Theta^{\prime} \right\rangle \right),
\end{equation} from which we get the rotated action set \( \mathcal{X}^{\prime}=\left\{\left( \hat{U}, \hat{U}_{\perp}\right)^{\top} X \left(\hat{V}, \hat{V}_{\perp}\right) : X \in \mathcal{X}\right\} \). We then rearrange $\mathrm{vec}(\Theta^{\prime})$ such that the first $k\triangleq(d_1+d_2-r)r$ elements are in the subspace, while the remaining elements are in the orthogonal complement, which is explicitly given by
\begin{equation}\label{eq4}
\boldsymbol{\theta}^*=\left(\mathrm{vec}\left( \Theta_{1:r,1:r}^\prime\right);\mathrm{vec}\left( \Theta_{r+1:d_1,1:r}^\prime\right);\mathrm{vec}\left( \Theta_{1:r,r+1:d_2}^\prime\right);\mathrm{vec}\left( \Theta_{r+1:d_1,r+1:d_2}^\prime\right)\right),
\end{equation} 
and denote the corresponding action set as \begin{equation}\label{eq6}
    \mathcal{X}_{\text {vec }}^{\prime}=\left\{ \left(\mathrm{vec}\left( X_{1:r,1:r}^{\prime}\right);\mathrm{vec}\left( X_{r+1:d_1,1:r}^{\prime}\right);\mathrm{vec}\left( X_{1:r,r+1:d_2}^{\prime}\right);\mathrm{vec}\left( X_{r+1:d_1,r+1:d_2}^{\prime}\right)\right): X^{\prime} \in \mathcal{X}^{\prime} \right\}.
\end{equation}

Then it is not hard to see that the orthogonal complements $\hat{U}_{\perp}, \hat{V}_{\perp}$ of the estimated subspaces can be used to assess the redundancy in the transformed parameter $\boldsymbol{\theta}^*$. Specifically, we have
\begin{equation*}
    \begin{aligned}
       \left\|\boldsymbol{\theta}_{k+1:d_1 d_2}^*\right\|^2  =&\sum_{i>r , j>r} \Theta_{i j}^{\prime 2}=\left\|\hat{U}_{\perp}^{\top}\left(U^* S^* V^{* \top}\right) \hat{V}_{\perp}\right\|_F^2  \leq\left\|\hat{U}_{\perp}^{\top} U^*\right\|_F^2\left\|S^*\right\|^2\left\|\hat{V}_{\perp}^{\top} V^*\right\|_F^2\\ \leq &\left\|\hat{U}_{\perp}^{\top} U^*\right\|_F^2
\left\|\hat{V}_{\perp}^{\top} V^*\right\|_F^2 , 
    \end{aligned}
\end{equation*}
where $U^* \in \mathbb{R}^{d_1 \times r}$, $S^* \in \mathbb{R}^{r \times r}$, and $V^* \in \mathbb{R}^{d_2 \times r}$ are obtained from the truncated SVD of the true parameter matrix $\Theta^* = U^* S^* V^{*\top}$. The above inequality shows that the more accurate the estimated subspaces are, the tail of the transformed parameter vector $\boldsymbol{\theta}^*$ tends to 0, indicating that the transformed parameter $\boldsymbol{\theta}^*$ lies in an almost low-dimensional space. This dimensionality reduction effectively reduces the complexity of the parameter space and lays the foundation for designing an efficient action selection strategy within this almost low-dimensional space.

\subsection{Graph-Based Action Selection Strategy}\label{sec3.2}

\begin{algorithm}[t]
\caption{Graph-LowGLM-UCB} 
\hspace*{0.02in} {\bf Input:} 
$\mathcal{X}_{\text {vec }}^{\prime}, T_2, k, L, \alpha, \delta,\tau, \lambda_2, \lambda_{\perp}, c_\mu, a_\mu, k_\mu, \omega$.
\begin{algorithmic}[1]
\State Let $\Lambda=\operatorname{diag}\left(\lambda_2, \ldots, \lambda_2, \lambda_{\perp}, \ldots, \lambda_{\perp}\right)$, where $\lambda_2$ occupies the first $k$ diagonal entries; $V(c_\mu)=\frac{\Lambda+a_\mu \alpha \widetilde{X}^\top L \widetilde{X}}{c_\mu}$;  $V_1(c_\mu)=V(c_\mu)+\sum_{i=1}^{T_1} \boldsymbol{x}_{h, i} \boldsymbol{x}_{h, i}^{\top}$, where the data $\left\{\boldsymbol{x}_{h, i}\right\}_{i=1}^{T_1}$ collected
in the subspace estimation. \label{algo2_line1}
\For{$t=1$ to $T_2$} \label{algo2_line2}
\State Compute $\hat{\boldsymbol{\theta}}_t$ according to Equation \eqref{eq7}. \label{algo2_line3}
\State Choose action $\boldsymbol{x}_t=\operatorname{argmax}_{\boldsymbol{x} \in \mathcal{X}_{\text {vec }}^{\prime}} \left\{\mu \left(\langle \hat{\boldsymbol{\theta}},\boldsymbol{x} \rangle \right)+\frac{k_\mu}{c_\mu} e_t \|\boldsymbol{x}\|_{V_t^{-1}(c_\mu)} \right\}$ and receive reward $y_{t}$, where $e_t=\omega \sqrt{\log \frac{\left|V_t(c_\mu)\right|}{|V(c_\mu)| \delta^2}}+\sqrt{c_\mu}(\sqrt{\lambda_2} +\sqrt{\lambda_{\perp}} \tau+1)$. \label{algo2_line4}
\State Update  
$V_{t+1}(c_\mu)=V_t(c_\mu)+\boldsymbol{x}_t \boldsymbol{x}_t^{\top}$. \label{algo2_line5}
\EndFor \label{algo2_line6}
\end{algorithmic}
\label{algo2}
\end{algorithm}

In the second stage (lines 7-9) of proposed Algorithm \ref{algo1}, we shall design an efficient action selection strategy named as Graph-LowGLM-UCB for the transformed parameter space, which integrates graph information under the framework of classical upper confidence bound (UCB) algorithm. The main objective of such UCB algorithm is to construct a confidence interval for the expected reward of each candidate action and select the action with the highest upper confidence bound. One can refer to Algorithm \ref{algo2} for specific details.

In the transformed low-dimensional space, the observation model can be equivalently reformulated as $y_t=\mu (\langle \boldsymbol{x}_t,\boldsymbol{\theta}^*\rangle)+\epsilon_t$, where the action $\boldsymbol{x}_t \in \mathcal{X}_{\text {vec }}^{\prime}$ and the noise $\epsilon_t$ is an independent sub-Gaussian variable with parameter $\omega$. The forms of the true transformed parameter $\boldsymbol{\theta}^*$ and the corresponding action set $\mathcal{X}_{\text {vec}}^{\prime}$ have been given in \eqref{eq4} and \eqref{eq6}, respectively. Considering that $\boldsymbol{\theta}^*$ is unknown in practice, it is necessary to provide a good estimate of it.


Importantly, the first stage of the proposed Algorithm \ref{algo1} provides prior information that $\|\boldsymbol{\theta}_{k+1:d_1d_2}^*\| \leq \tau$, where $\tau$ is the error bound of the subspace estimation. Using this prior, we can explicitly separate the informative low-rank part $\boldsymbol{\theta}_{1:k}^*$ from the redundant residual part $\boldsymbol{\theta}_{k+1:d_1d_2}^*$ to obtain a better estimation of unknown $\boldsymbol{\theta}^*$. Specifically, we introduce a diagonal matrix $\Lambda = \mathrm{diag}(\lambda_2, \dots, \lambda_2, \lambda_\perp, \dots, \lambda_\perp)$, which imposes different levels of penalization on the two aforementioned parts of the parameter $\boldsymbol{\theta}^*$. The first $k$ diagonal entries are set to $\lambda_2$, while the remaining entries are set to $\lambda_\perp$. Building on this, we further incorporate the previously introduced Laplacian regularization term defined in Equation \eqref{eq1}, leading to the final form of the minimization problem (maximum likelihood with regularizations), which is formulated as \begin{equation}\label{eq7}
\begin{aligned}   \hat{\boldsymbol{\theta}}_t=\arg \min_{\boldsymbol{\theta} \in \mathbb{R}^{d_1 d_2}} \sum_{i=1}^{T_1} &\left[b (\langle \boldsymbol{x}_{h, i},\boldsymbol{\theta}\rangle)-y_{h, i} \langle \boldsymbol{x}_{h, i},\boldsymbol{\theta}\rangle \right]+ \sum_{i=1}^{t-1} \left[b (\langle \boldsymbol{x}_i,\boldsymbol{\theta}\rangle)-y_i \langle \boldsymbol{x}_i,\boldsymbol{\theta}\rangle \right] \\ & +\frac{1}{2}\|\boldsymbol{\theta}\|_{\Lambda}^2+\frac{a_\mu \alpha}{2} \boldsymbol{\theta}^\top \widetilde{X}^\top L \widetilde{X} \boldsymbol{\theta}. 
\end{aligned}
\end{equation}
Here, we reuse the historical data from stage 1, $\left\{\boldsymbol{x}_{h, i}, y_{h, i}\right\}_{i=1}^{T_1}$, to increase the sample size in stage 2, thereby improving the accuracy of parameter estimation $\hat{\boldsymbol{\theta}}_t$.

Based on the above parameter estimation $\hat{\boldsymbol{\theta}}_t$, we select actions according to the UCB principle, as described in line 4 of Algorithm~\ref{algo2}. Specifically, for each candidate action $\boldsymbol{x}$, we construct an optimistic estimate of its reward defined as $\mu \left(\langle \hat{\boldsymbol{\theta}}_t, \boldsymbol{x} \rangle \right) + w_t$, where the first term $\mu \left(\langle \hat{\boldsymbol{\theta}}_t, \boldsymbol{x} \rangle \right)$ denotes the estimated expected reward, capturing the exploitation behavior based on historical observations, while the second term $w_t$ represents the confidence width, quantifying the uncertainty of the estimate. We then select the action with the highest optimistic estimate from the available candidates, that is,
\begin{equation}\label{ucb}
    \boldsymbol{x}_t = \operatorname{argmax}_{\boldsymbol{x} \in \mathcal{X}_{\text{vec}}^{\prime}} \left\{ \mu \left(\langle \hat{\boldsymbol{\theta}}_t, \boldsymbol{x} \rangle \right) + w_t \right\},
\end{equation}
where the explicit form of $w_t$ will be derived in the subsequent theoretical analysis.

\section{Theory}\label{sec4}



In this section, we shall present a theoretical analysis in terms of the cumulative regret bound for the proposed algorithm, highlighting the significant role of graph information for improving such a bound. To this end, we begin by presenting the key and mild assumptions required for the analysis. We then establish the subspace estimation error bound with graph information. Next, we derive an explicit expression of the confidence width $w_t$ in \eqref{ucb}, which leads to the regret bound for the graph-based action selection strategy, i.e., Algorithm \ref{algo2}. Finally, we combine these results to derive the overall cumulative regret bound of the proposed Algorithm \ref{algo1}. 

\subsection{Assumptions} 

In this subsection, we would like to introduce several necessary assumptions to facilitate subsequent theoretical analysis.



\begin{assumption}[Feature Distribution]\label{as1}
There exists a sampling distribution $\mathbb{D}$ over $\mathcal{X}$ with its associated density $\mathbf{P}=\left(p_{i j}\right): \mathbb{R}^{d_1 \times d_2}$, such that $\forall i,j$, the score function $S(\cdot)$ satisfies $\mathbb{E}\left[S^2(X_{i j}) \right] \leq \gamma$, where the random matrix $X$ is drawn from $\mathbb{D}$, that is, the probability density of $X_{i j}$ is $p_{i j}$. The explicit form of $S(\cdot)$ is given in Definition \ref{s-function-def}.

\end{assumption}

\begin{assumption}[Bounded Norm]\label{as2}
The true parameter and feature matrices have a bounded norm, that is, $\|\Theta^*\|_F \leq 1, \|X\|_F\leq1$, where $ \forall X \in \mathcal{X}$.
\end{assumption}

\begin{assumption}[The Inverse Link Function]\label{as3}
The inverse link function $\mu(\cdot)$ is continuously differentiable and there exist two constants $c_\mu, k_\mu$ such that $0<c_\mu \leq \mu^{\prime}(x) \leq k_\mu$ for all $|x| \leq 1$.
\end{assumption}


Now we shall state that the above assumptions can be easily satisfied in practice. Specifically, Assumption \ref{as1} only requires the score function to have a finite second moment, which can be met by many common distributions such as Gaussian, exponential, and uniform distributions. Assumption \ref{as2}, a boundedness condition commonly used in the bandits literature (e.g., \cite{bastani2020online,lu2021low}), can be easily satisfied by normalization. Assumption \ref{as3} is a standard assumption in generalized liner bandits~\citep{kang2022efficient,yi2024effective}, which only requires the first derivative of the inverse link function to be bounded. It is easy to check that the common generalized linear models, such as logistic and Poisson regression models, satisfy this condition. 



\subsection{Deriving Error Bound for Subspace Estimation with Graph Information}
When applying existing CB analytical framework to derive error bound for parameter estimation, we observe that this framework is typically designed for a single regularization term. However, this work involves two regularization terms capturing
the low-rank structure and graph information, respectively. Inspired by \cite{rao2015collaborative}, we equivalently represent the sum of these two regularization terms as a single atomic norm. To this end, we introduce the following lemma as a theoretical foundation for our analysis.



\begin{lemma}\label{lem1}
Let $L_h$ and $L_w$ denote the Laplacian matrices corresponding to the user and item similarity graphs, respectively, and let $L$ be the Laplacian matrix corresponding to the (user, item) pair similarity graph. Suppose that $L_w = U_w S_w U_w^\top$ and $L_h = U_h S_h U_h^\top$ are the full SVD of $L_w$ and $L_h$, respectively. Define the atomic set as $$\mathcal{A} \triangleq \left\{ \boldsymbol{\omega}_i \boldsymbol{h}_i^\top : \boldsymbol{\omega}_i = A \boldsymbol{u}_i, \boldsymbol{h}_i = B \boldsymbol{v}_i, \|\boldsymbol{u}_i\| = \|\boldsymbol{v}_i\| = 1, A = U_w S_w^{-1/2}, B = U_h S_h^{-1/2} \right\}.$$ Assume that conditions $L_h = \lambda I, 2a_\mu \alpha \left(H^\top \otimes I \right) \widetilde{X}^\top L \widetilde{X} \left(H \otimes I \right) + \lambda I = I \otimes L_w$, are satisfied, where $W = U_\Theta S_\Theta^{\frac{1}{2}}, H=V_\Theta S_\Theta^{\frac{1}{2}}$, $\Theta = U_\Theta S_\Theta U_\Theta^\top$ is the full SVD of $\Theta$, then the weighted atomic norm given in Definition \ref{def5} can be rewritten as $\|\Theta\|_{\mathcal{A}} = \inf_{\Theta} \left\{ \lambda \|\Theta\|_* + a_\mu \alpha \mathrm{tr}\left\{ \Theta^{\top} \widetilde{X}^\top L \widetilde{X}  \Theta \right\} \right\}$.
 
\end{lemma}




With the above lemma, the estimator $\widehat{\Theta}_{T_1}$ in Equation \eqref{eq2} can be equivalently obtained by
\begin{equation}\label{eq8}
  \widehat{\Theta}_{T_1}=\arg \min _{\Theta \in \mathbb{R}^{d_1 \times d_2}} \langle\Theta, \Theta\rangle-\frac{2}{T_1} \sum_{i=1}^{T_1}\left\langle \psi_\nu\left(y_i \cdot S\left(X_i\right)\right), \Theta\right\rangle
  +\beta\|\Theta\|_{\mathcal{A}},
\end{equation}
where $\beta$ is the tuning parameter to control the influence of its corresponding regularization term. Then, we can further apply the existing CB analytical framework by separately
handling the loss function and the single regularization term. This leads to the following theoretical bound.



\begin{theorem}[Error Bound for Parameter Estimation with Graph Information]



Under Assumptions~\ref{as1}-\ref{as3}, and the conditions of Lemma \ref{lem1}. Consider the estimator $\widehat{\Theta}_{T_1}$ in Equation \eqref{eq2} with $\lambda \leq 1$ and $\alpha \leq \frac{1 -\lambda}{4 a_\mu n (n-1)}$, or equivalently, the estimator $\widehat{\Theta}_{T_1}$ in Equation \eqref{eq8} with 
$$\nu = \sqrt{\frac{2 \log \left(2\left(d_1 + d_2\right) / \delta \right)}{\left( 4 \omega^2 + r_{\max}^2 \right) \gamma T_1 \max \{ d_1, d_2 \} }}, 
\beta = 4 \zeta \sqrt{\frac{2 \left(4 \omega^2 + r_{\max}^2 \right) \gamma d_1 d_2 \log \left(2\left(d_1 + d_2\right) / \delta \right)}{T_1}}.
$$ 
Then, with probability at least $1 - \delta$, the estimation error satisfies
$$
\left\|\widehat{\Theta}_{T_1} - \mu^* \Theta^*\right\|_F^2 \leq \frac{c_1 \zeta^2 d_1 d_2 \gamma r \log \left(\frac{2\left(d_1 + d_2\right)}{\delta}\right)}{T_1},
$$
where $c_1 = 36 \left( 4 \omega^2 + r_{\max}^2 \right) $, $r_{\max}=|\mu(0)|+k_\mu$, $\mu^* = \mathbb{E}[\mu^\prime (\langle X, \Theta^* \rangle)]$, and the constant $\zeta \in (0,1]$ is related to the graph information.
\label{the1}
\end{theorem}

Building on the above parameter estimation result, we further perform the full SVD on the estimated parameter as $\widehat{\Theta}_{T_1} = (\hat{U}, \hat{U}_{\perp}) \hat{S} (\hat{V}, \hat{V}_{\perp})^\top$ to quantify the estimation error of the low-rank parameter subspace. 


\begin{corollary}[Error Bound for Subspaces Estimation with Graph Information] 

Suppose $\widehat{\Theta}_{T_1}$ is obtained from Equation~\eqref{eq2} or Equation~\eqref{eq8} as an estimate of the true parameter $\Theta^*$. The matrices $U^*$ and $V^*$ are obtained from the truncated SVD of $\Theta^* = \hat{U^*}S^* V^{*\top}$. Under the conditions of Theorem~\ref{the1}, with probability at least $1 - \delta$, we have $$  \left\|\hat{U}_{\perp}^\top U^*\right\|_F\left\|\hat{V}_{\perp}^\top V^*\right\|_F \leq \frac{\left\|\mu^* \Theta^*-\hat{\Theta}\right\|_F^2}{c_r^2} \leq \frac{c_1 \zeta^2 d_1 d_2 \gamma r \log \left(\frac{2\left(d_1+d_2\right)}{\delta}\right)}{T_1 c_r^2},$$ 
where $c_r>0$ denotes the lower bound of the $r$-th singular value of $\Theta^*$ and $c_1$ represents some constant.

\label{cor1}
\end{corollary}

Based on the above result, we achieve a more accurate estimation of the low-rank subspace. Compared to the error bound in \cite{kang2022efficient} that ignores the graph information, our bound incorporates a factor $\zeta$ smaller than 1 in the numerator, which leads to a significant reduction in estimation error. This improvement highlights the importance of integrating graph information into the low-rank subspace estimation process.



\subsection{Analyzing Regret Bound for Graph-Based Action Selection Strategy} 



In analyzing the cumulative regret of an action selection strategy, it is essential to first explicitly characterize the confidence width of the estimated reward.

\begin{theorem}[The Confidence Width]

Let $ V(s) = \frac{\Lambda + a_\mu \alpha \widetilde{X}^\top L \widetilde{X}}{s}, V_t(s) = V(s) + \sum_{i=1}^{T_1} \boldsymbol{x}_{h,i} \boldsymbol{x}_{h,i}^\top + \sum_{k=1}^{t-1} \boldsymbol{x}_k \boldsymbol{x}_k^\top$. Assume that $\|\boldsymbol{\theta}^*\| \leq 1$ and $\|\boldsymbol{\theta}_{k+1:d_1d_2}^*\| \leq \tau$. Then, under Assumptions \ref{as1}-\ref{as3}, with probability at least $1 - \delta$, for all $t \geq 0$ and any action $\boldsymbol{x} \in \mathcal{X}_{\text{vec}}^{\prime}$, we have $ \left| \mu\left( \langle \boldsymbol{x}, \boldsymbol{\theta}^* \rangle \right) - \mu\left( \langle \boldsymbol{x}, \hat{\boldsymbol{\theta}}_t \rangle \right) \right| \leq w_t$, where $ w_t =\frac{k_\mu}{c_\mu} \|\boldsymbol{x}\|_{V_t^{-1}(c_\mu)} \left[\omega \sqrt{ \log \left( \frac{ |V_t(c_\mu)| }{ |V(c_\mu)| \delta^2 } \right) } +  \sqrt{c_\mu} \left( \sqrt{\lambda_2} + \sqrt{\lambda_\perp} \tau + 1 \right)\right]   
$. 

\label{the2}
\end{theorem}

Building on the above theorem, we further address the term $\log \left( \frac{ |V_t(c_\mu)| }{ |V(c_\mu)| } \right) $. Due to the incorporation of graph information, the matrix $V(c_\mu)$ is no longer diagonal as in Lemma 2 of \cite{jun2019bilinear}, making their result inapplicable. To address this issue, we extend the result to the general positive definite matrix and derive the regret by carefully choosing $\lambda_\perp$. 

\begin{theorem}[The Regret of Graph-Based Action Selection]
Under Assumptions \ref{as1}-\ref{as3}, the regret of Graph-LowGLM-UCB with $\lambda_{\perp}=\frac{c_\mu T}{k \log (1+\frac{c_\mu T}{\lambda_2})}$ is $\tilde{\mathcal{O}}\left( \omega k \sqrt{T} +  \tau T\right)$, with probability at least $1-\delta$.
\label{the4}
\end{theorem}

The regret in the above theorem depends on the effective dimension $k = (d_1 + d_2 - r)r$ instead of the original dimension $d_1 d_2$, resulting in a better bound. This improvement comes from explicitly separating the informative low-rank part from the redundant residual part, highlighting the benefit of transforming the problem into an almost low-dimensional space.

\subsection{Deriving Overall Regret Bound}


By combining the theoretical results of both preceding subsections, we derive the overall regret of the proposed Algorithm \ref{algo1}.

\begin{theorem}[Overall Regret]   
Under Assumptions \ref{as1}-\ref{as3}, suppose we run Algorithm \ref{algo1} with
$T_1 \geq  \frac{ \zeta \sqrt{d_1 d_2 \gamma r T} }{c_r}$, $\lambda_{\perp}=\frac{c_\mu T}{k \log (1+\frac{c_\mu T}{\lambda_2})}$, $\tau=\frac{c_1 \zeta^2 d_1 d_2 \gamma r \log \left(\frac{2\left(d_1+d_2\right)}{\delta}\right)}{T_1 c_r^2}$, the overall regret is $\tilde{\mathcal{O}} \left(\frac{  \zeta \sqrt{d_1 d_2 \gamma r T} }{c_r} \right)$,
with probability at least $1-\delta$.
\label{the5}
\end{theorem}

Compared with existing methods, our regret offers two main advantages: (1) Compared to method that ignores graph information~\citep{kang2022efficient}, our bound introduces a factor $\zeta$ in the numerator, which decreases with richer graph information, thereby reducing cumulative regret; (2) Compared to method that does not utilize low-rank structure~\citep{yang2020laplacian}, our bound is tighter for dimension $d_2$ and is independent of the number of users.

\section{Experiments}\label{sec5}


In this section, we
demonstrate the superior performance of our proposed Algorithm \ref{algo1}, GG-ESTT, through numerical experiments on both synthetic data and real-world data. All experimental evaluations were conducted on a computer equipped with an Intel(R) Xeon(R) Gold 5120 CPU @ 2.20GHz, using Python 3.9.

\subsection{Synthetic Data Experiments}

We generate a real parameter matrix $\Theta^*$ with dimensions $d_1 \times d_2$ as $\Theta^* = A M B^\top$, where $M = U V^\top$, $U \in \mathbb{R}^{d_1 \times r}$, $V \in \mathbb{R}^{d_2 \times r}$, and the elements are independently sampled from $N(0,1)$. Specifically, the definitions of $A$ and $B$ can be found in Lemma \ref{lem1}. Regarding the action set $\mathcal{X}$, there are two generation methods:
\begin{itemize}
     \item The impact of graph information richness: directly drawing $n$ matrices with dimensions $d_1 \times d_2$, where the elements are independently sampled from $N(0, 1)$. 
    \item Comparison with related algorithms: drawing $n_1$ and $n_2$ vectors independently from $N(\boldsymbol{0}, I_{d_1})$ and $N(\boldsymbol{0}, I_{d_2})$, respectively, to serve as feature vectors for item 1 and item 2. Subsequently, by computing the outer product of these vectors, we generate $n=n_1 n_2$ feature matrices, each corresponding to a (item 1, item 2) pair.
   
\end{itemize}
For each action $X_t \in \mathcal{X}$, the rewards are generated using three common generalized linear models:
\begin{itemize}
    \item Linear case: $y_t \sim N(\langle X_t, \Theta^* \rangle, \omega^2)$, where $\omega=0.01$.
    \item Binary logistic case: $y_t \sim Logistic(p_t)$, where $p_t=\frac{1}{1+e^{-\langle X_t, \Theta^* \rangle}}$.
    \item Poisson case: $y_t \sim Poisson(\eta_t)$, where $\eta_t=e^{\langle X_t, \Theta^* \rangle}$.
\end{itemize}
For each simulation setting, we repeated the process 5 times to calculate the average regret at each time step and its standard deviation confidence interval.

\subsubsection{The Impact of Graph Information Richness}

\begin{figure}[t]
\FIGURE
{
\subcaptionbox{Linear case.\label{fig:figure8_test1}
}
{\includegraphics[width=0.35\textwidth]{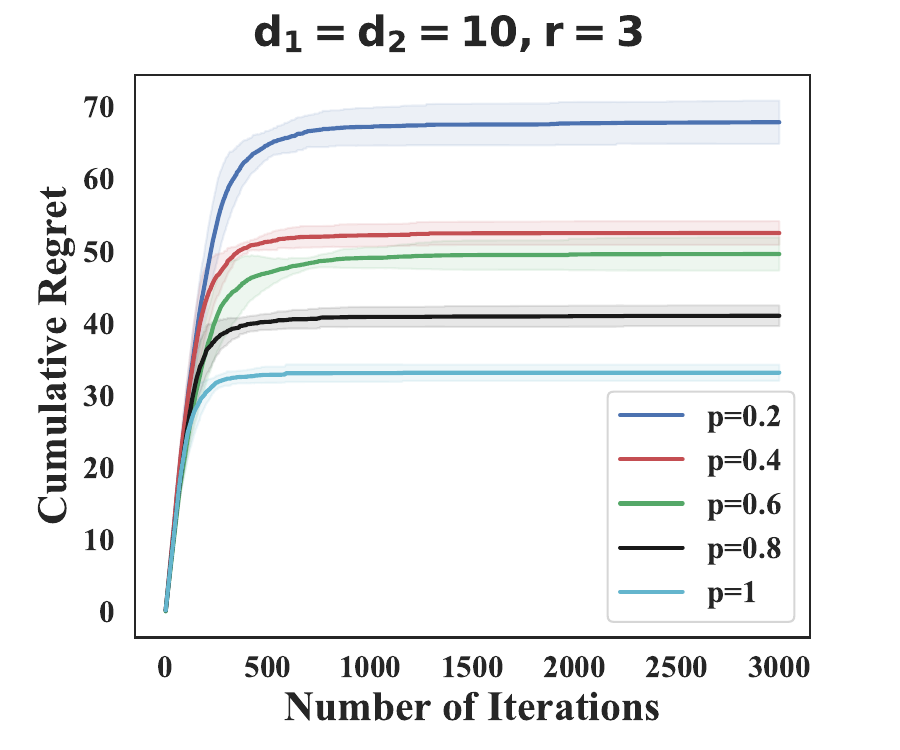}}
\hspace{-0.5cm}\subcaptionbox{Binary logistic case.\label{fig:figure8_test2}
}
{\includegraphics[width=0.35\textwidth]{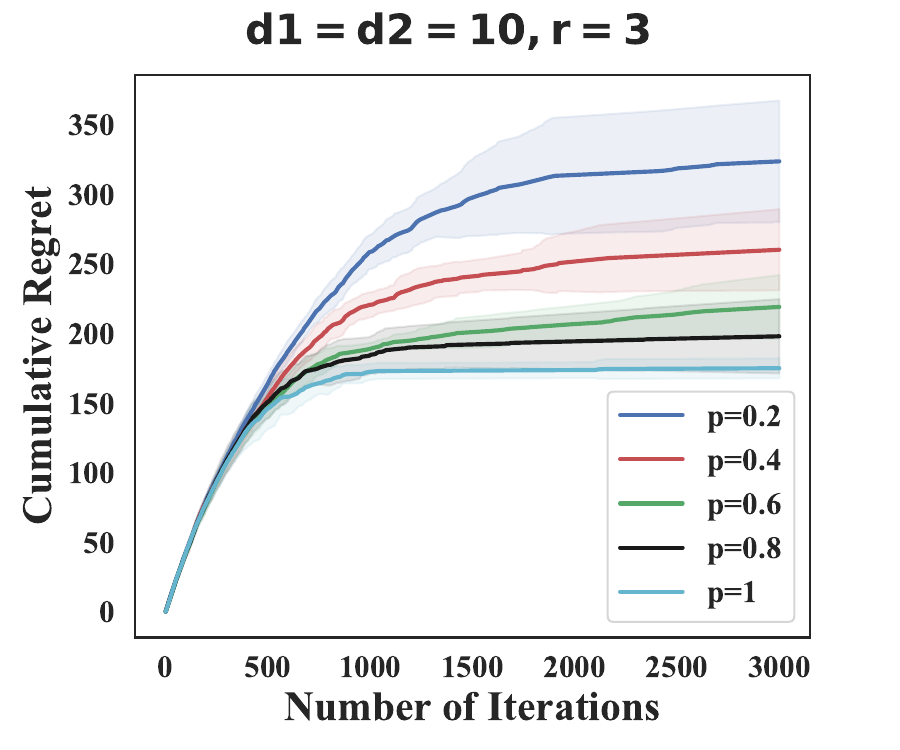}}
\hspace{-0.5cm}\subcaptionbox{Poisson case.\label{fig:figure8_test3}
}
{\includegraphics[width=0.35\textwidth]{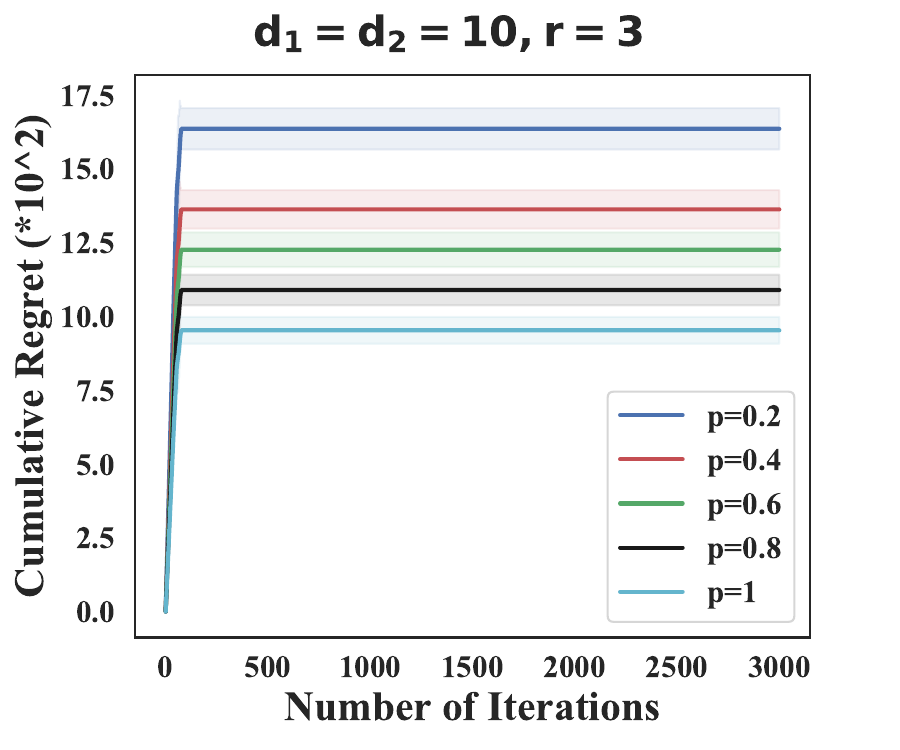}}
}
{
ER Random Graph: the Impact of Graph Information Richness.
\label{fig:figure8}}
{
}
\end{figure}
\setlength{\floatsep}{2mm}
\begin{figure}[t]
\FIGURE
{
\subcaptionbox{Linear case.\label{fig:figure9_test1}
}
{\includegraphics[width=0.35\textwidth]{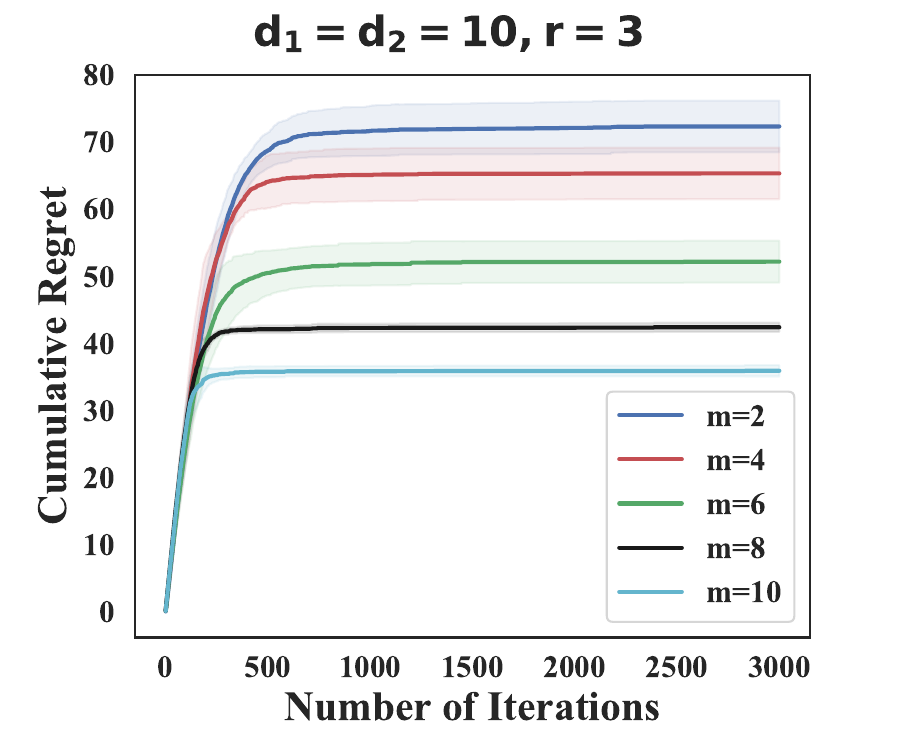}}
\hspace{-0.5cm}\subcaptionbox{Binary logistic case.\label{fig:figure9_test2}
}
{\includegraphics[width=0.35\textwidth]{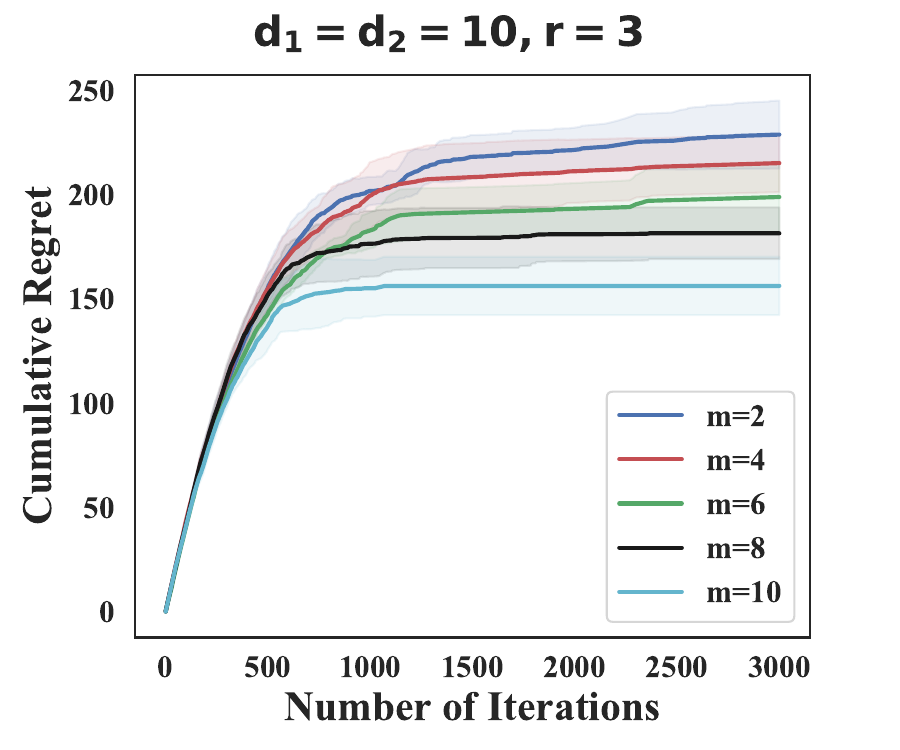}}
\hspace{-0.5cm}\subcaptionbox{Poisson case.\label{fig:figure9_test3}
}
{\includegraphics[width=0.35\textwidth]{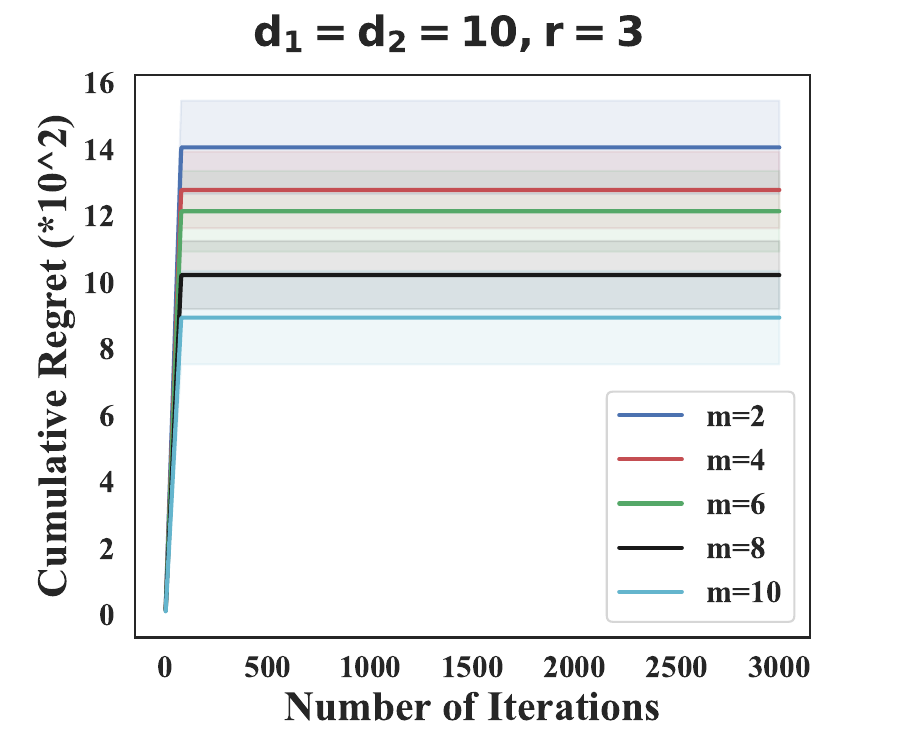}}
}
{
BA Random Graph: the Impact of Graph Information Richness.
\label{fig:figure9}}
{
}
\end{figure}

To better demonstrate the algorithm's ability to utilize graph information, we examine how the increase of similarity information in the graph affects the cumulative regret, under fixed the number of actions and parameter matrix settings. To this end, we construct the graph $\mathcal{G}$ using two representative random graph models: (1) Erdős–Rényi (ER) random graph: Each pair of nodes is independently connected with probability $p$. (2) Barabási–Albert (BA) random graph: Starting from an initial connected graph with $m$ nodes, new nodes are added to the graph one at a time, each forming $m$ edges to existing nodes. The connections follow the principle of preferential attachment, whereby nodes with more edges have a greater probability of being selected for connection by the new node. As shown in Figure~\ref{fig:figure8} (ER random graph) and Figure~\ref{fig:figure9} (BA random graph), when other settings are fixed, the cumulative regret decreases as the parameter of the random graph model (i.e., $p$ or $m$) increases. This trend indicates that larger graph model parameter lead to richer similarity information in the graph, which facilitates more accurate subspace estimation and faster identification of the best arms, thereby improving the overall decision-making performance. This empirical result is consistent with our theoretical analysis, where graph information plays a key role in reducing regret, further validating the effectiveness of the proposed algorithm in leveraging graph information.


\subsubsection{Comparison with State-of-the-Art Algorithms}

In this subsection, we compare the proposed algorithm with several representative and publicly available bandit
algorithms, including those that focus on low-rank structure, those that
utilize graph information, and classical vector-based algorithms without any additional structural information. The comprehensive list is as follows.




 
\begin{itemize}
    \item  GG-ESTT (this paper): Algorithm \ref{algo1}.
    \item LowESTR~\citep{lu2021low}: the matrix UCB algorithm based on low-rank subspace exploration.
    \item G-ESTT~\citep{kang2022efficient}: the generalized linear matrix UCB algorithm based on low-rank subspace exploration.
    \item GG-OFUL (Algorithm \ref{algo3} in Appendix \ref{appendixg}):  the vectorized version of the proposed GG-ESTT. \item GraphUCB~\citep{yang2020laplacian}: the graph based UCB algorithm via regularization.
    \item CLUB~\citep{gentile2014online}: the graph based UCB algorithm via clustering.
    \item UCB-GLM~\citep{li2017provably}: the generalized linear version of the UCB algorithm.
    \item NeuralUCB~\citep{zhou2020neural}: the neural network based UCB algorithm.
\end{itemize}

\begin{figure}
\centering
\FIGURE
{
\subcaptionbox{ER random graph.\label{fig:figure5_test1}
}
{\includegraphics[width=0.25\textwidth]{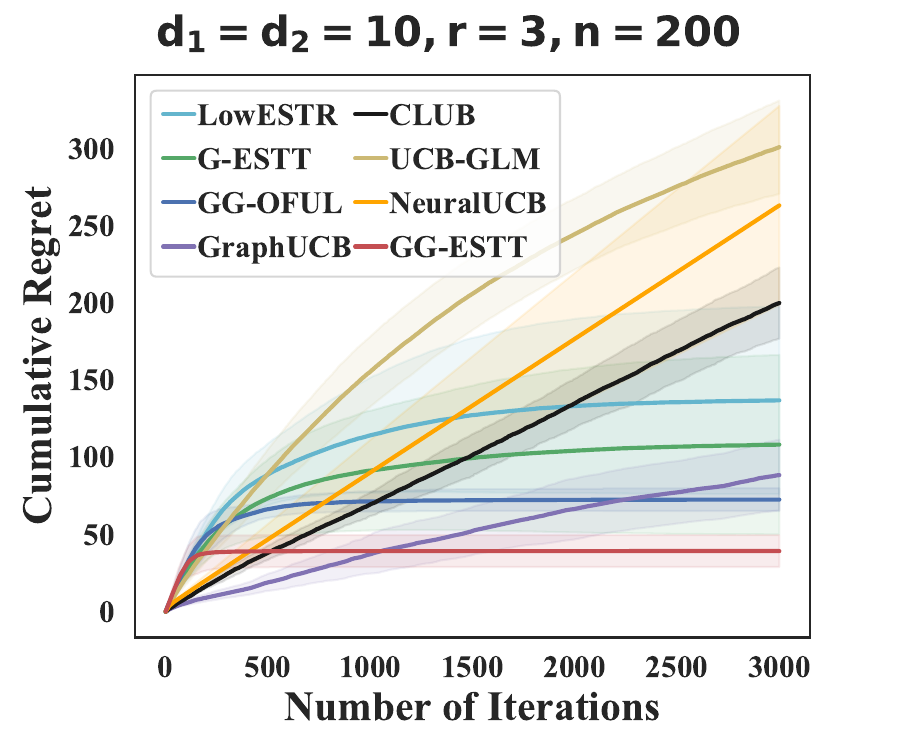}}
\hspace{-0.3cm}
\subcaptionbox{BA random graph.\label{fig:figure5_test2}
}
{\includegraphics[width=0.25\textwidth]{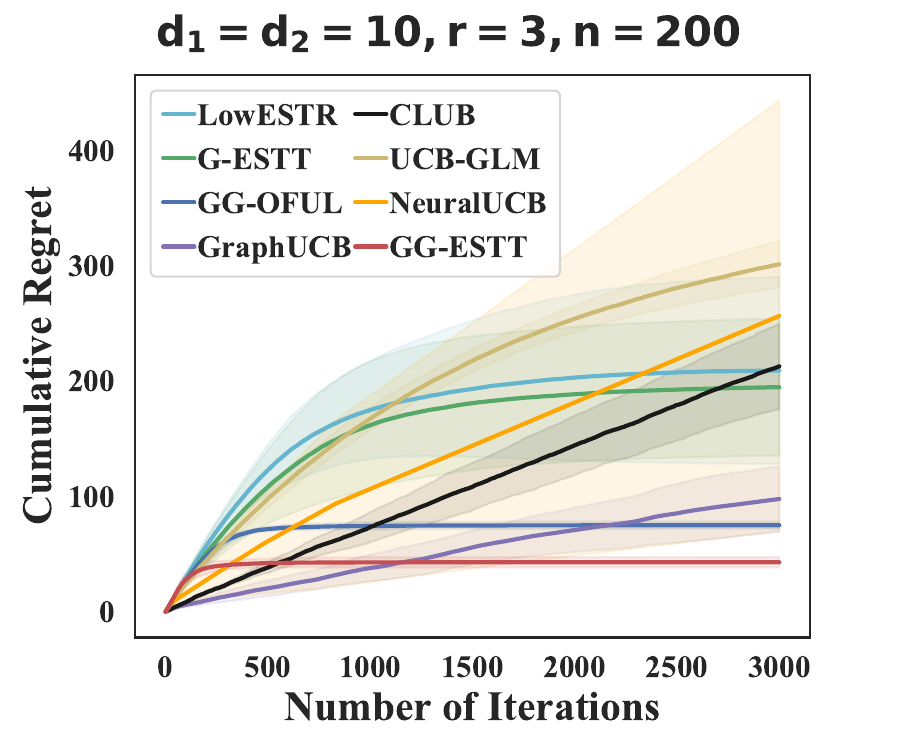}}
\hspace{-0.3cm} 
\subcaptionbox{ER random graph.\label{fig:figure5_test3}
}
{\includegraphics[width=0.25\textwidth]{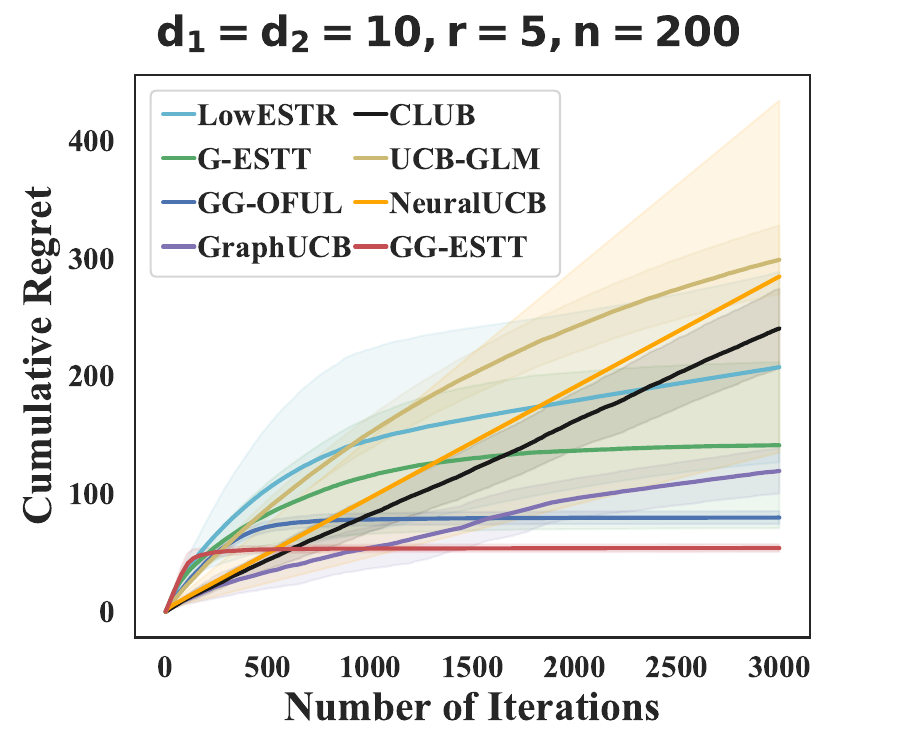}}
\hspace{-0.3cm}
\subcaptionbox{ER random graph.\label{fig:figure5_test4}
}
{\includegraphics[width=0.25\textwidth]{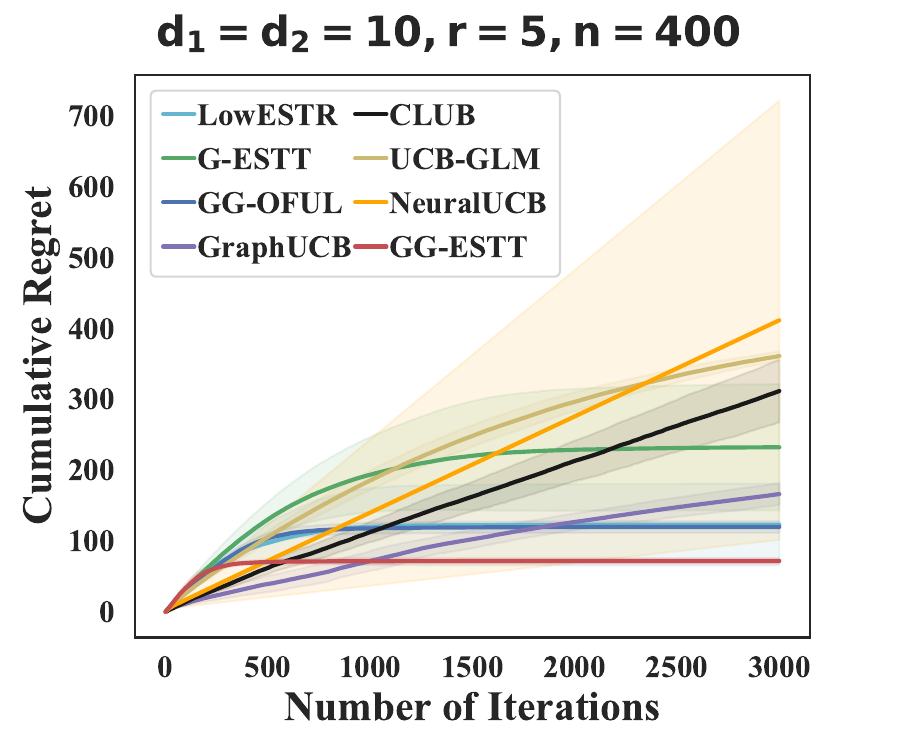}}
}
{
Linear Case: Comparison of Regret in Four Different Settings about $r, n$, and Random Graph Model.
\label{fig:figure5}
}
{
}
\end{figure}
\setlength{\floatsep}{2mm}
\begin{figure}
\centering
\FIGURE
{
\subcaptionbox{ER random graph.\label{fig:figure6_test1}
}
{\includegraphics[width=0.25\textwidth]{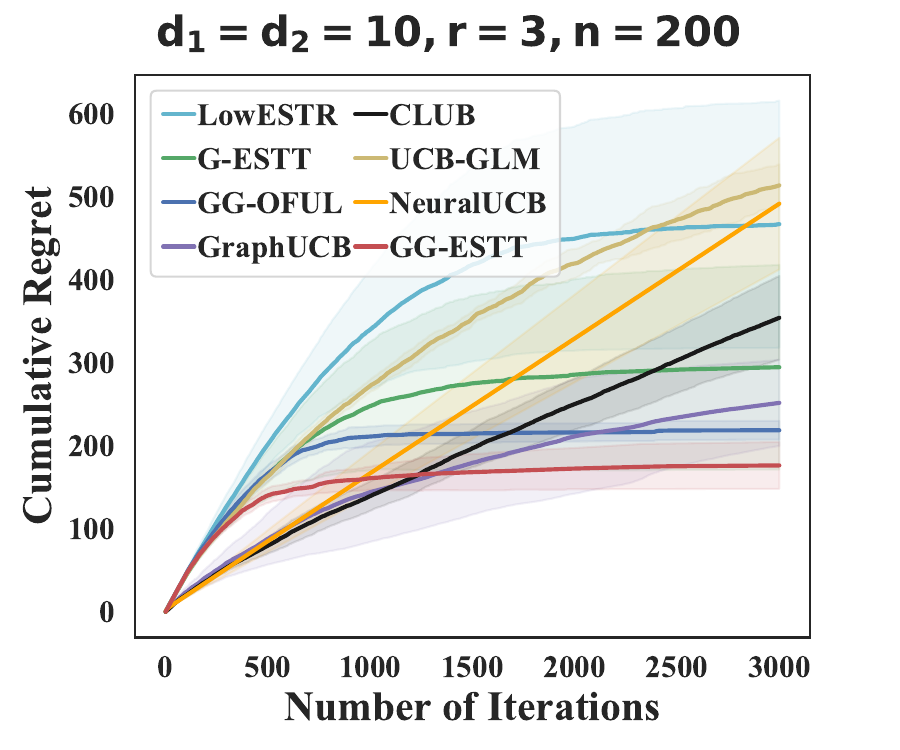}}
\hspace{-0.3cm}
\subcaptionbox{BA random graph.\label{fig:figure6_test2}
}
{\includegraphics[width=0.25\textwidth]{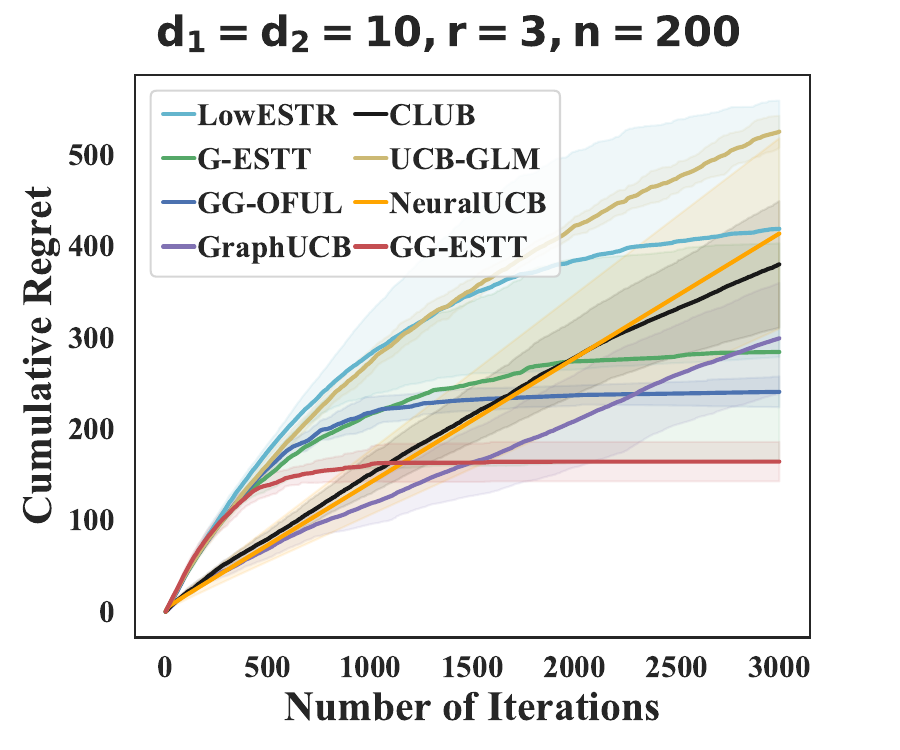}}
\hspace{-0.3cm} 
\subcaptionbox{ER random graph.\label{fig:figure6_test3}
}
{\includegraphics[width=0.25\textwidth]{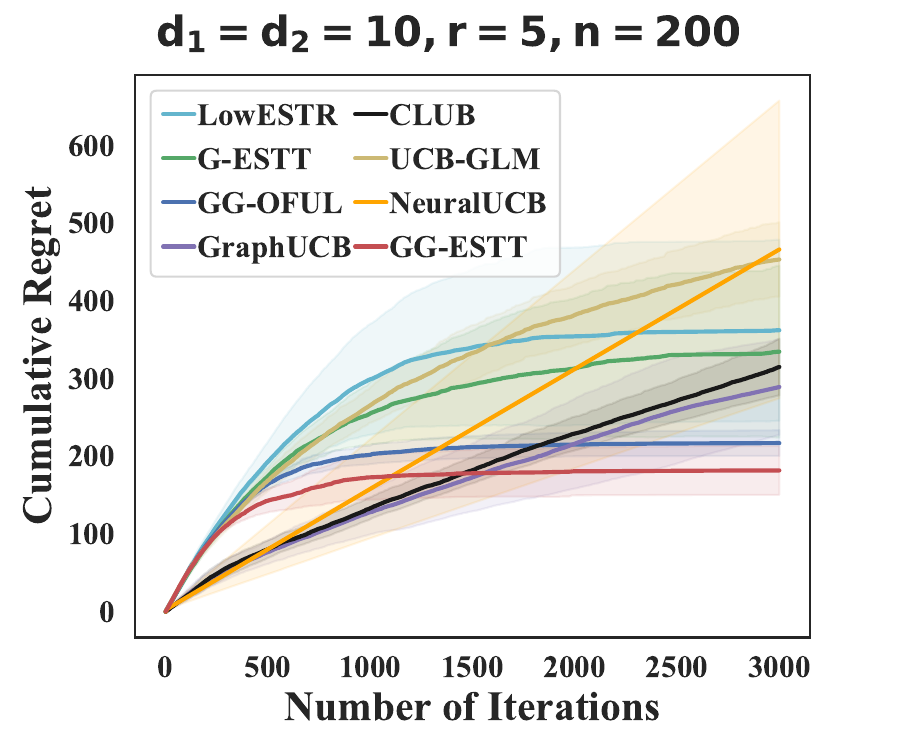}}
\hspace{-0.3cm}
\subcaptionbox{ER random graph.\label{fig:figure6_test4}
}
{\includegraphics[width=0.25\textwidth]{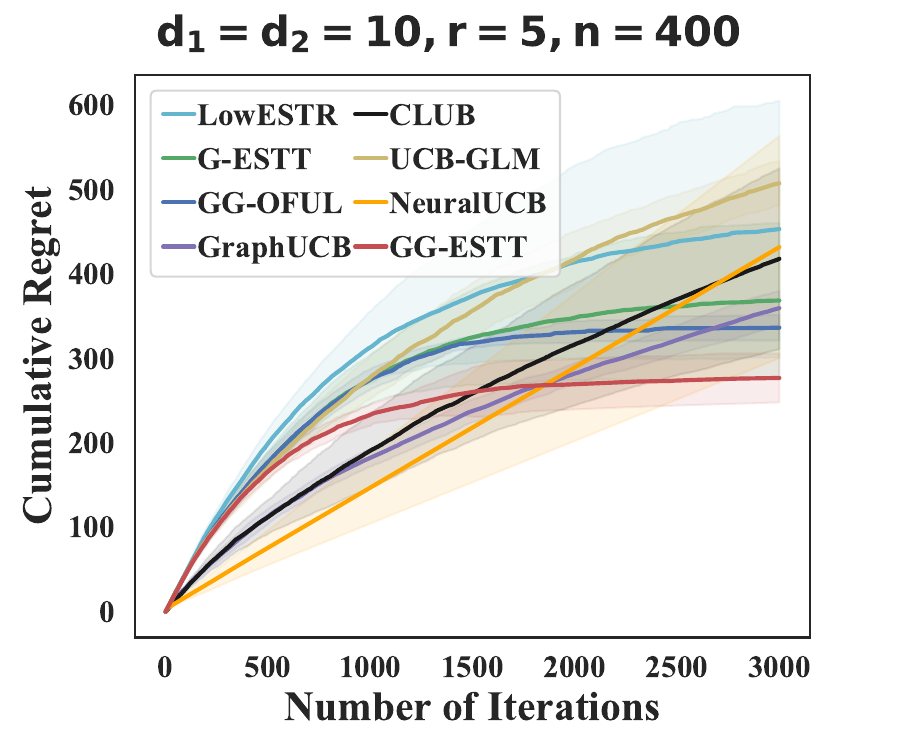}}
}
{
Binary Logistic Case: Comparison of Regret in Four Different Settings about $r, n$, and Random Graph Model.
\label{fig:figure6}
}
{
}
\end{figure}
\setlength{\floatsep}{2mm}
\begin{figure}[h]
\centering
\FIGURE
{
\subcaptionbox{ER random graph.\label{fig:figure7_test1}
}
{\includegraphics[width=0.25\textwidth]{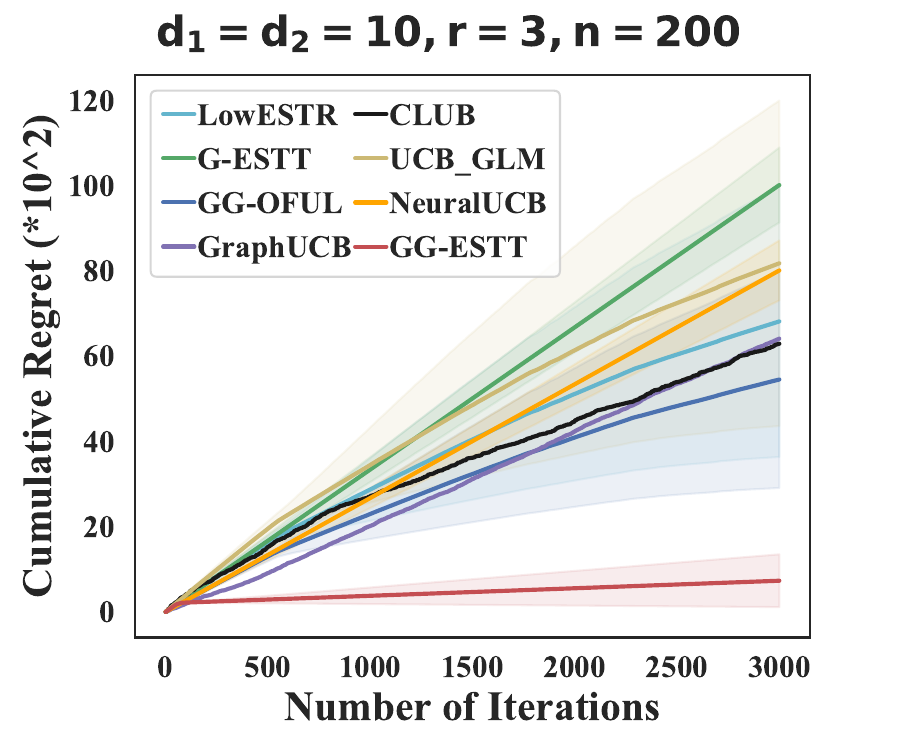}}
\hspace{-0.3cm}
\subcaptionbox{BA random graph.\label{fig:figure7_test2}
}
{\includegraphics[width=0.25\textwidth]{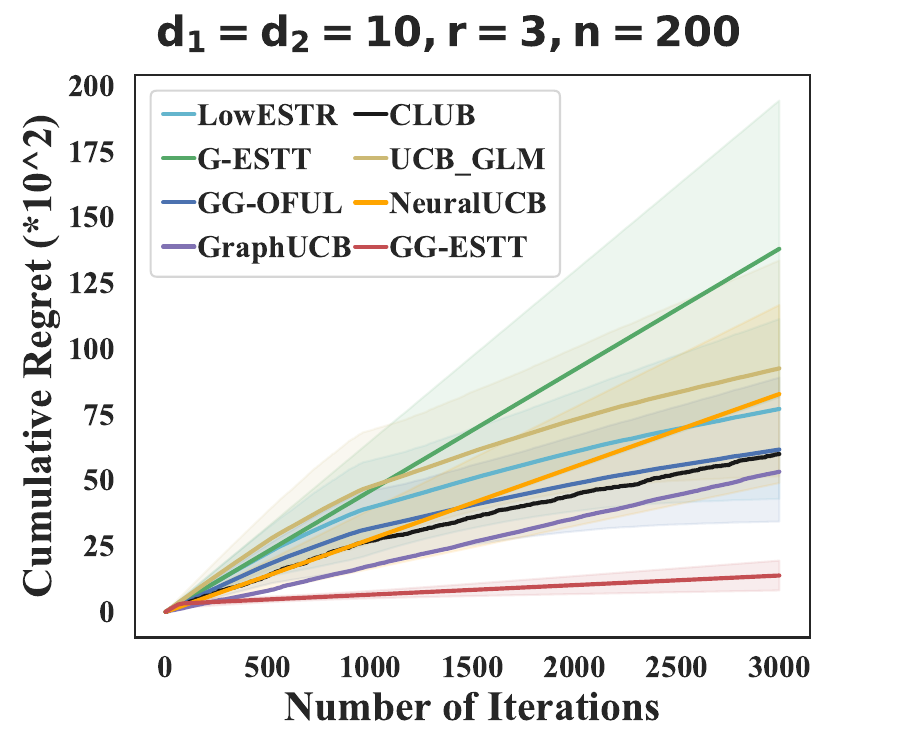}}
\hspace{-0.3cm} 
\subcaptionbox{ER random graph.\label{fig:figure7_test3}
}
{\includegraphics[width=0.25\textwidth]{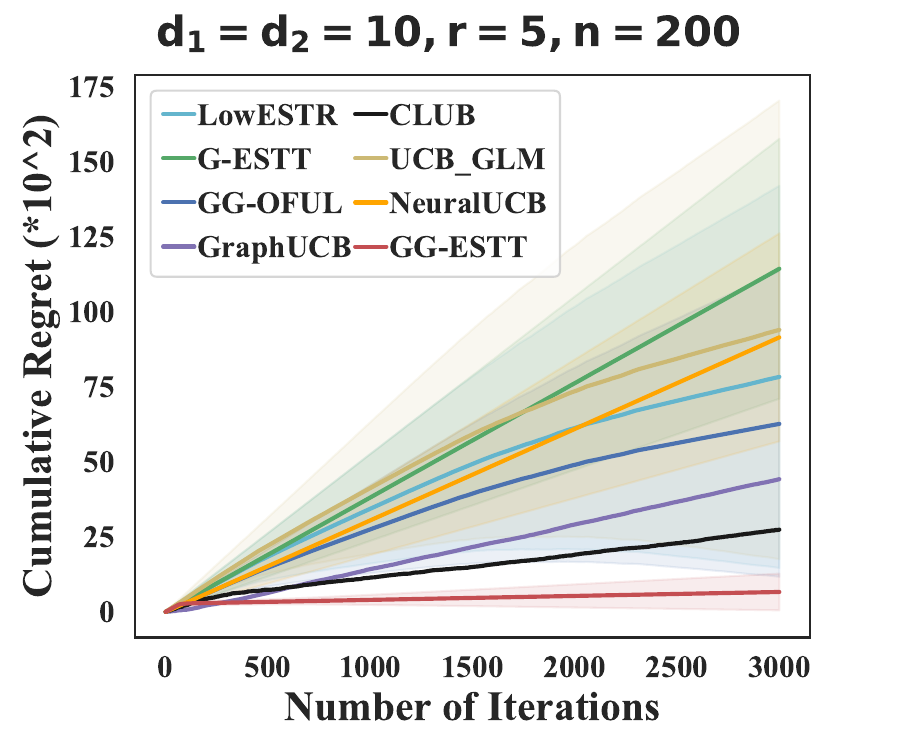}}
\hspace{-0.3cm}
\subcaptionbox{ER random graph.\label{fig:figure7_test4}
}
{\includegraphics[width=0.25\textwidth]{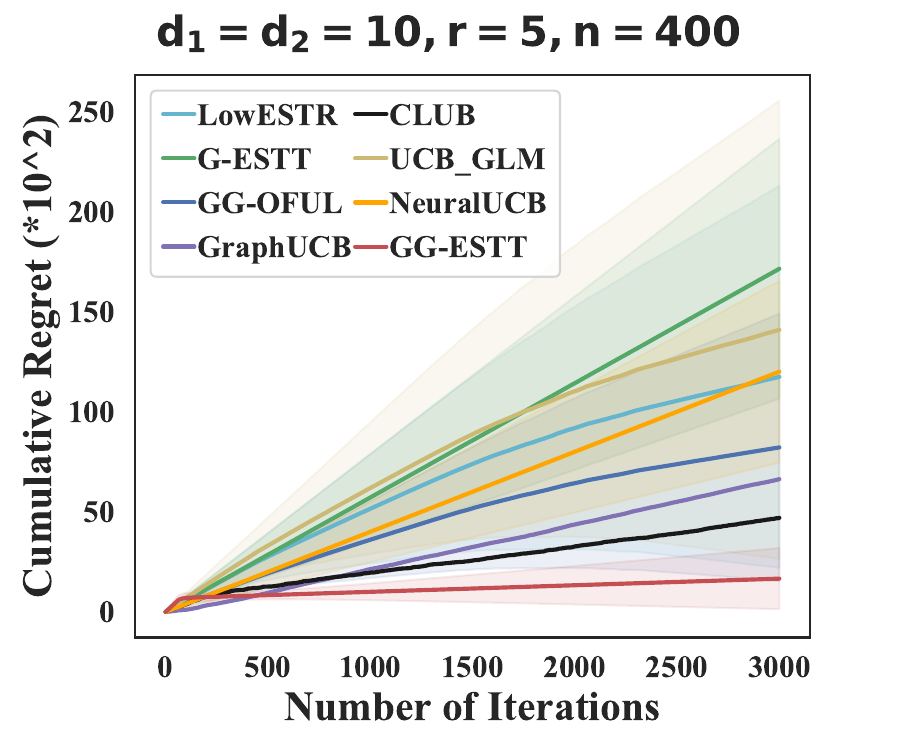}}
}
{
Poisson Case: Comparison of Regret in Four Different Settings about $r, n$, and Random Graph Model.
\label{fig:figure7}
}
{
}
\end{figure}

In this set of simulations, we consider four different settings that vary in terms of the rank $r$ of the parameter matrix $\Theta^*$, the number of actions $n$, and the type of random graph model used. These experimental results are shown in Figure \ref{fig:figure5} (linear case), Figure \ref{fig:figure6} (binary logistic case) and Figure \ref{fig:figure7} (poisson case). Clearly, in these four different settings, our proposed GG-ESTT algorithm outperforms other algorithms in terms of cumulative regret. Additionally, we observed that our algorithm achieves faster to obtain sub-linear performance. Compared to the two algorithms in the low-rank matrix setting, these advantages indicate that the first stage obtained a more accurate estimation of the parameter subspace. This is consistent with our theory, suggesting that GG-ESTT not only leverages the low-rank structure, but also integrates graph information, thus improving the accuracy of estimating the parameter $\Theta^*$ and providing better guidance for subsequent decisions. Compared to algorithms in the vector setting, this advantage suggests that data vectorization impairs its inherent low-rank structure, affecting the effectiveness of the estimation and subsequently influencing the action selection process.

\subsection{Real-World Data Experiments
}



To better highlight the outstanding performance of our algorithm, we consider three real-world applications involving cancer treatment (linear case), movie recommendation (binary logistic case), and ad searches (poisson case). To enhance the persuasiveness of our experiments, we use publicly available datasets for each application and apply the following preprocessing steps:


\begin{itemize}
    \item The Cancer Cell Line Encyclopedia (CCLE) dataset\endnote{See \url{https://depmap.org/portal/download/}.}: We perform data normalization due to significant differences in data ranges. We consider the (cell line, target) pair as an action, involving 20 cell lines and 17 targets, resulting in 340 actions. The reward is given by the corresponding drug sensitivity measurements. 
    \item The MovieLens dataset\endnote{See \url{https://grouplens.org/datasets/movielens/}.}: We address missing values by imputing them with zeros. We treat the (user, movie) pair as an action. To minimize the proportion of missing data, we select the top 20 active users and movies, forming 400 actions. The reward is set to 1 if the user's rating for the movie is greater than 3, and 0 otherwise.
    \item The KDD Cup 2012 dataset\endnote{See \url{http://www.kddcup2012.org/c/kddcup2012-track2/}.}: We address missing values by imputing them with zeros. We treat the (ad, search keywords) pair as an action,  involving 20 ads and 13 search keywords, resulting in 260 actions. The reward is the ad's impression, which refers to the total number of times the ad has been viewed.
\end{itemize}


\begin{figure}
\FIGURE
{
\subcaptionbox{CCLE dataset.\label{fig:figure11_test1}
}
{\includegraphics[width=0.35\textwidth]{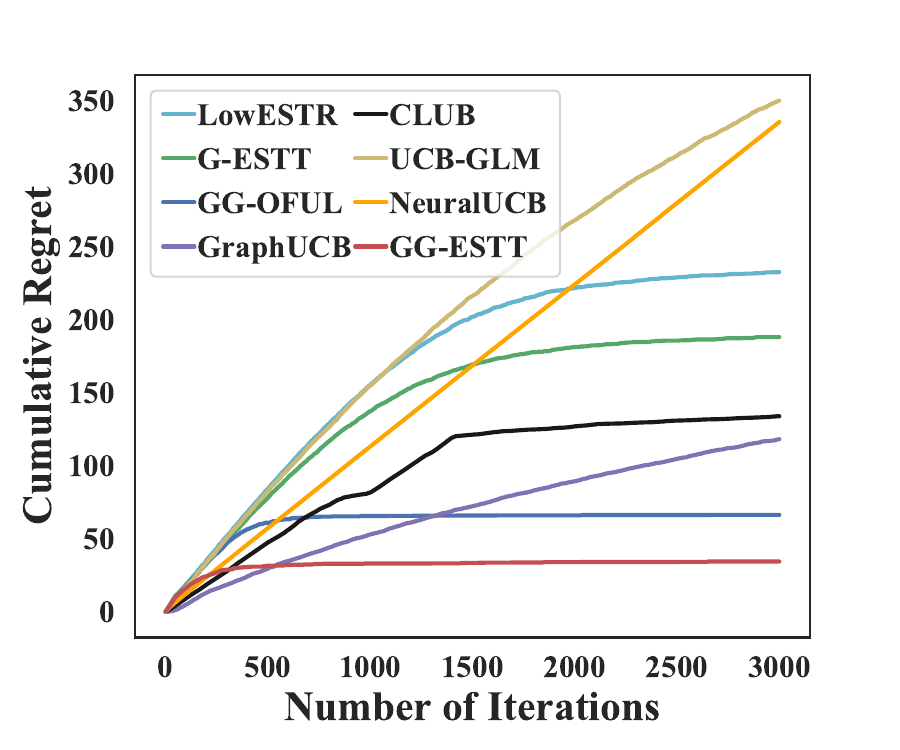}}
\hspace{-0.5cm}\subcaptionbox{MovieLens dataset.\label{fig:figure11_test2}
}
{\includegraphics[width=0.35\textwidth]{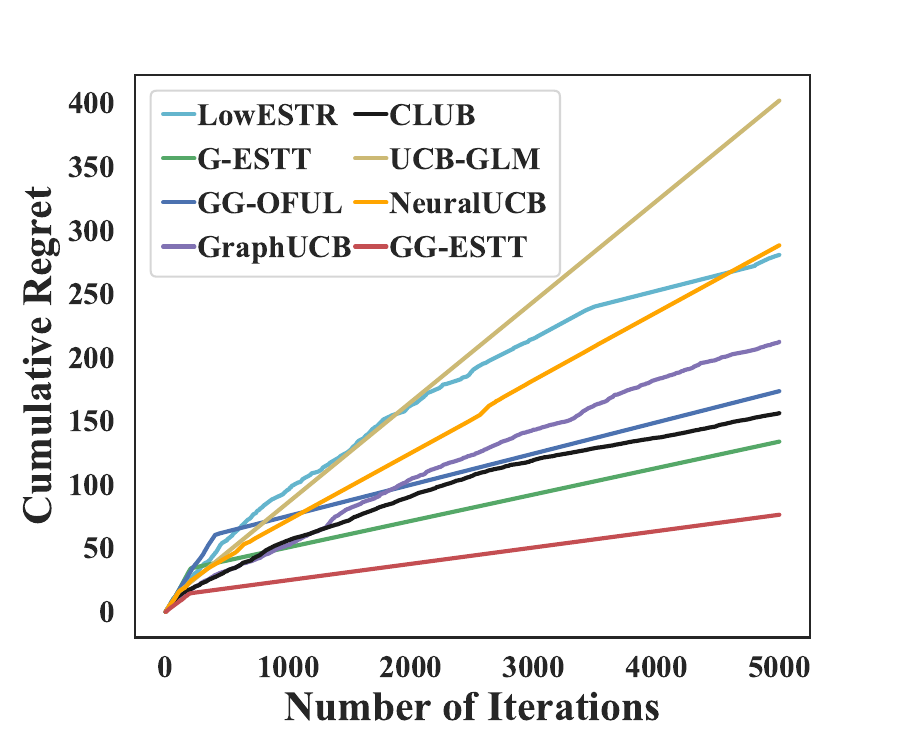}}
\hspace{-0.5cm}\subcaptionbox{KDD Cup 2012 dataset.\label{fig:figure11_test3}
}
{\includegraphics[width=0.35\textwidth]{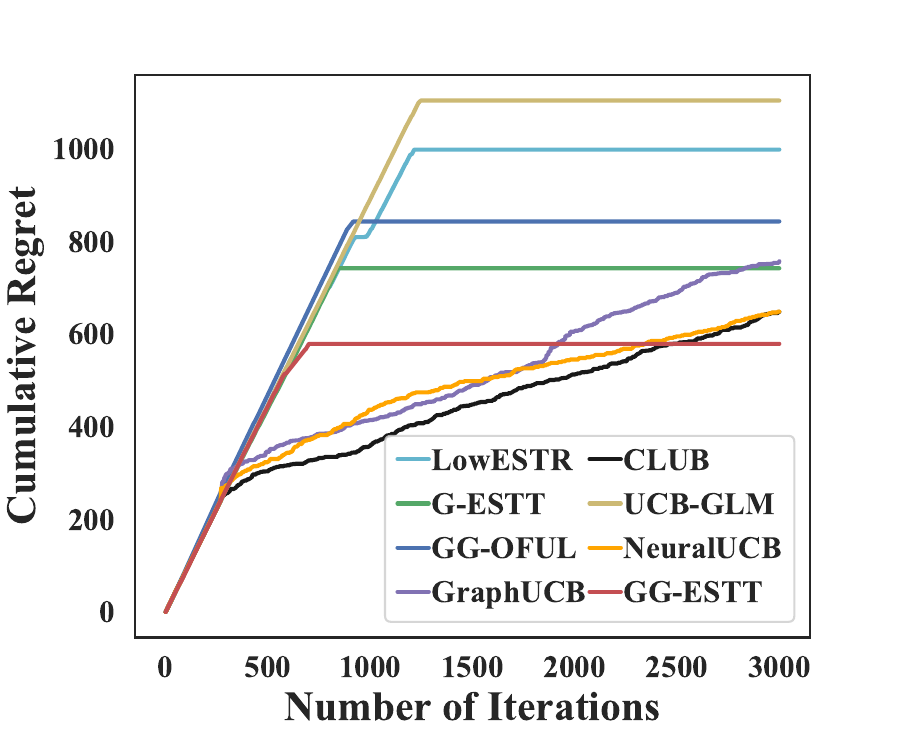}}
}
{
Comparison of Regret on Three Real-World Datasets.
\label{fig:figure11}}
{
}
\end{figure}

We model the rewards as the corresponding matrix $R$ and perform SVD of $R=F_1 \Theta^* F_2^\top$. We then define $\left\{ X_l=F_1[i,:] \odot F_2[j,:]: l = (i-1)n_2 + j \in \{1,\cdots,n\},i \in \{1,\cdots,n_1\}, j \in \{1,\cdots,n_2\}\right\}$ as the action set and $\Theta^*$ as the true parameter, where $n_1$ is the number of item 1, $n_2$ is the number of item 2, the number of (item 1, item 2) is $n \triangleq n_1 n_2$. In contrast to simulated experiments, for real-world data, we construct the corresponding graph $\mathcal{G}$ using the 5-nearest neighbor method.



The cumulative regret as a standard metric to measure the reward gap between the algorithm’s decisions and the optimal ones. In other words, a smaller cumulative regret indicates that the algorithm is making more accurate decisions. From Figure \ref{fig:figure11}, it is evident that over a certain period, the proposed GG-ESTT algorithm can achieve sublinear regret more rapidly, while also having the smallest cumulative regret. In other words, compared to other algorithms, our approach leverages the inherent structure (graph information and low-rankness) of the dataset more effectively, allowing for a better estimation of the parameter $\Theta^*$, and consequently, delivering superior performance.

\setlength{\floatsep}{2mm}

\begin{figure}
\FIGURE
{
\subcaptionbox{CCLE dataset.\label{fig:figure12_test1}
}
{\includegraphics[width=0.33\textwidth]{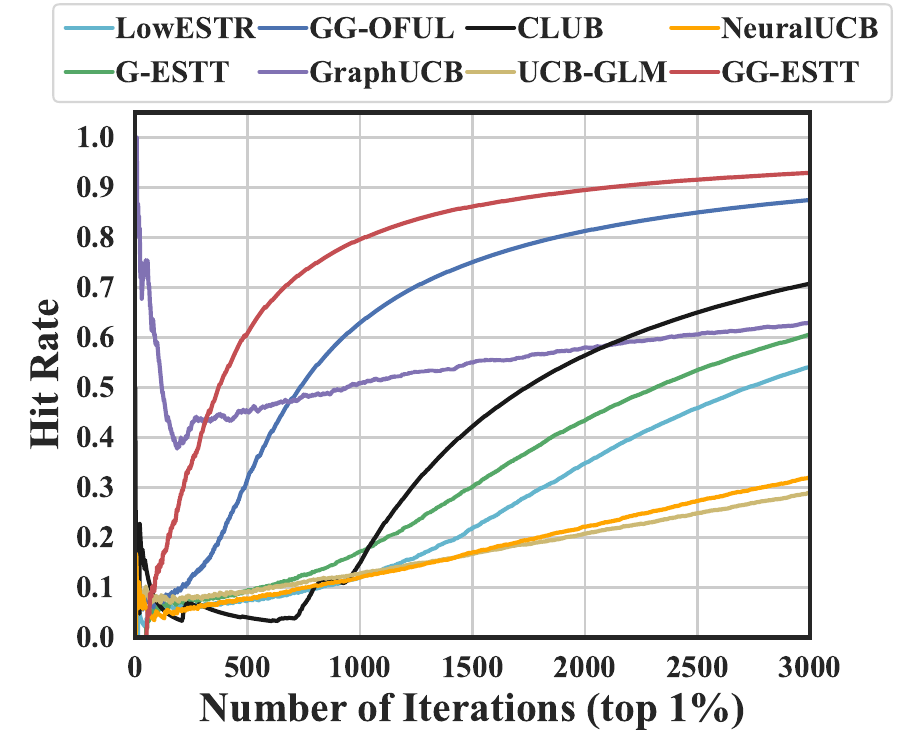}}
\hspace{-0.1cm}\subcaptionbox{MovieLens dataset.\label{fig:figure12_test2}
}
{\includegraphics[width=0.33\textwidth]{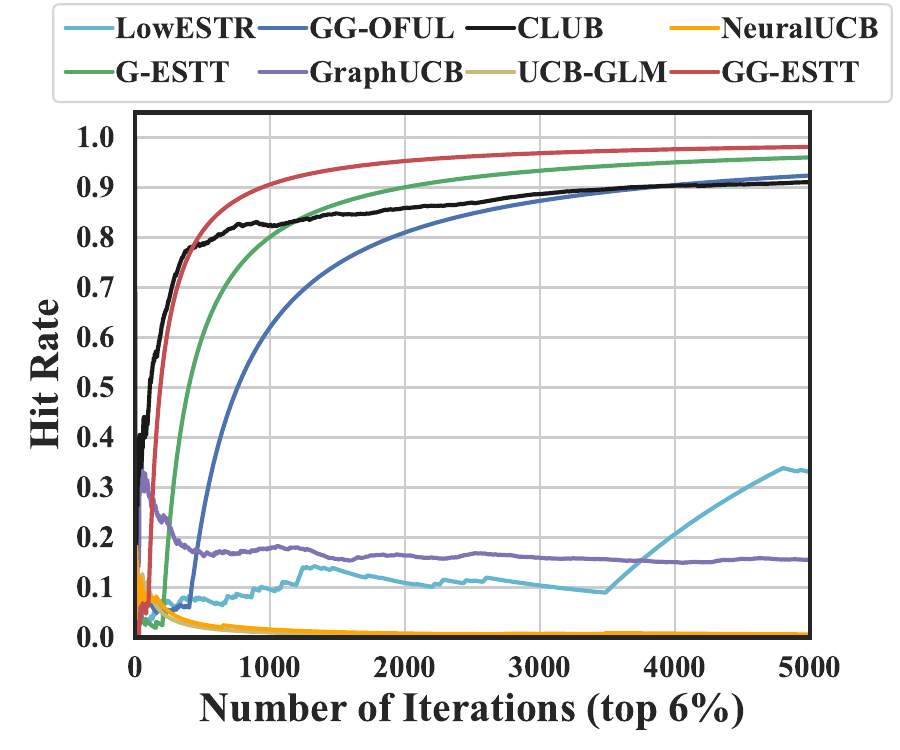}}
\hspace{-0.1cm}\subcaptionbox{KDD Cup 2012 dataset.\label{fig:figure12_test3}
}
{\includegraphics[width=0.33\textwidth]{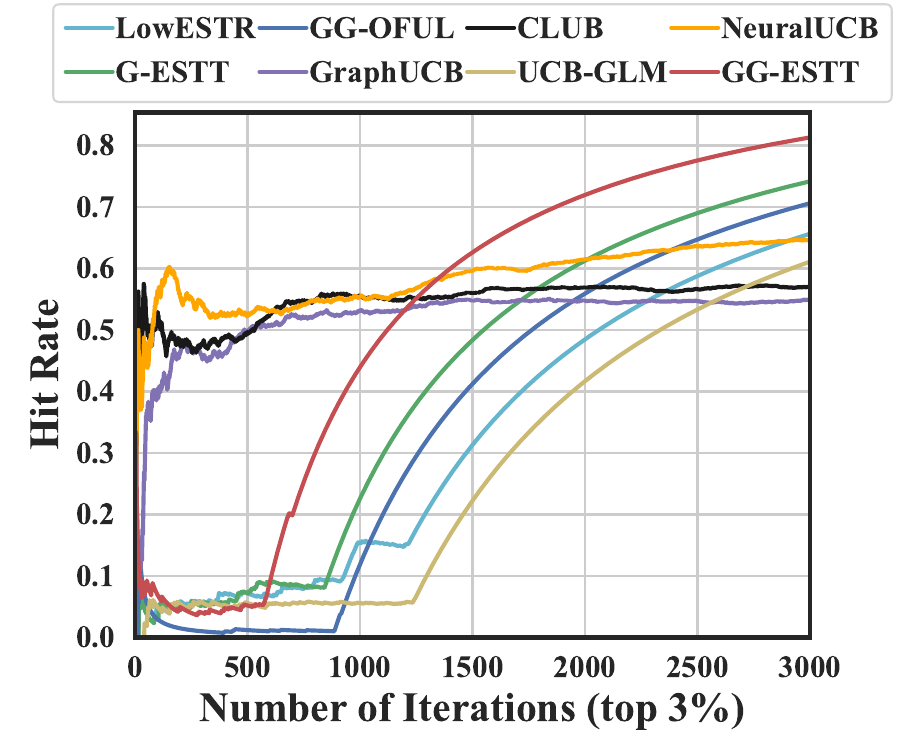}}
}
{
Comparison of Hit Rate on Three Real-World Datasets.
\label{fig:figure12}}
{}
\end{figure}

In addition, we further illustrate the performance of our algorithm by visualizing the Hit Rate~\citep{yi2024effective}, which quantifies the frequency of the optimal action selected over time and is defined as $\text{Hit Rate}(t) \triangleq \frac{1}{t} \sum_{i=1}^{t} \mathbb{I}(h_i \in \mathcal{H})$, 
where $h_i$ denotes the action selected at time $i$, and $\mathcal{H}$ represents the set of optimal actions (i.e., the set of actions ranked in the top $x \%$ in terms of rewards). As shown in Figure \ref{fig:figure12}, our algorithm consistently achieves the highest hit rate across all three real-world datasets. Moreover, the Hit Rate gradually approaches 1 as the number of iterations increases, further demonstrating the effectiveness of our algorithm in selecting optimal actions.

\section{Conclusion}\label{sec6}

This study tackles the practical needs of complex decision-making scenarios where both the low-rank matrix structure and graph information are available. Precisely, we introduce the nuclear norm and a newly designed graph Laplacian regularization to capture these two types of underlying structural information and devise a graph-based action selection strategy based on the UCB principle. By representing the sum of such two regularizers as an equivalent atomic norm, we are able to directly leverage existing analytical frameworks to derive a better cumulative regret bound compared with the popular alternatives. Specifically, the numerator of this bound includes a factor that reflects the richness of the graph information, which decreases as more graph information becomes available, thereby highlighting the significant role of utilizing graph information in decision-making. In future work, we are trying to extend this presented study to dynamic or unknown graph settings, thereby improving its applicability and generalization capabilities. 



\begingroup \parindent 0pt \parskip 0.0ex \def\enotesize{\normalsize} \theendnotes \endgroup

\bibliographystyle{informs2014} 
\bibliography{ref} 





  





\clearpage

\begin{APPENDICES}
\section*{Supplemental Material for ``Generalized Low-Rank Matrix Contextual  Bandits with Graph Information"}

This supplementary material provides the complete proofs of the theoretical results presented in the main manuscript, including Lemma~\ref{lem1} and Theorems~\ref{the1}–\ref{the5}, along with the corresponding technical lemmas required in the analysis. In addition, due to space limitations in the main text, we also include a detailed description of the vectorized version of the GG-ESTT introduced in Algorithm~\ref{algo3}.


\section{Proof of the Lemma \ref{lem1}}

\renewcommand{\thelemma}{A.\arabic{lemma}}  

\setcounter{lemma}{0}    



The Lemma \ref{lem1} represents the first key step in this work, providing an equivalent reformulation of the sum of two regularization terms as an atomic norm. To elaborate on this result, we introduce the following Lemma \ref{lem2}.
 
\begin{lemma}[Corollary 1 in \cite{rao2015collaborative}]\label{lemA.1}
  Let $L_w=U_w S_w U_w^\top, L_h=U_h S_h U_h^\top$ be the full SVD decomposition of $L_w, L_h$, respectively. Defined $\mathcal{A} \triangleq \{\omega_i h_i^\top: \omega_i=A u_i, h_i=B v_i, \|u_i \|=\|v_i \|=1,  A=U_w S_w^{-1 / 2}, B=U_h S_h^{-1 / 2}\}$, then $$\|Z\|_{\mathcal{A}}=\inf_{W,H} \frac{1}{2} \left\{ \mathrm{tr}(W^\top L_w W)+\mathrm{tr}(H^\top L_h H) \right\},$$ where $Z=W H^\top$. 
  \label{lem2}
\end{lemma}

\begin{proof}{Proof of Lemma \ref{lem1}}

Perform the full SVD of $\Theta$, denoted as $\Theta = U_\Theta S_\Theta V_\Theta^\top$. Define $W = U_\Theta S_\Theta^{1/2}$ and $H = V_\Theta S_\Theta^{1/2}$, then it follows that $\Theta = W H^\top$. Let $L_h=\lambda I,
2 a_\mu \alpha\left(H^\top \otimes I \right)\widetilde{X}^\top L \widetilde{X} \left(H \otimes I \right)+\lambda I=I \otimes L_w$, then the sum of the regularization terms in Equation \eqref{eq2} can be rewritten as
\begin{equation}
    \begin{aligned}
      & \lambda\|\Theta\|_* +a_\mu \alpha \mathrm{tr}\left\{ \Theta^{\top} \widetilde{X}^\top L \widetilde{X}  \Theta \right\}  \\
 =& \frac{\lambda}{2}  \left\{ \|W\|^2_F+\|H\|^2_F\right\}  +  a_\mu \alpha \mathrm{vec}^\top(WH^\top) \widetilde{X}^\top L \widetilde{X}  \mathrm{vec}(WH^\top)  \\
= & \frac{\lambda}{2}  \left\{ \mathrm{vec}^\top(W)\mathrm{vec}(W)+\mathrm{vec}^\top(H)\mathrm{vec}(H)\right\} +a_\mu \alpha \mathrm{vec}^\top(W)[\left(H^\top \otimes I \right)\widetilde{X}^\top L \widetilde{X} \left(H \otimes I \right)] \mathrm{vec}(W)  \\ 
= & \frac{1}{2}  \mathrm{vec}^\top(H)(\lambda I)\mathrm{vec}(H)+ \mathrm{vec}^\top(W)\left(a_\mu \alpha\left(H^\top \otimes I \right)\widetilde{X}^\top L \widetilde{X} \left(H \otimes I \right)+\lambda I\right) \mathrm{vec}(W) \\
= & \frac{1}{2}  \mathrm{vec}^\top(H)(I \otimes L_h)\mathrm{vec}(H)+ \frac{1}{2} \mathrm{vec}^\top(W)(I \otimes L_w) \mathrm{vec}(W) \\
= & \frac{1}{2}  \mathrm{tr}\left\{H^\top L_h H\right\}+ \frac{1}{2} \mathrm{tr}\left\{W^\top L_w W\right\},  
    \end{aligned}
\end{equation}
where the $i$-th row of the matrix $\widetilde{X}$ corresponds to the vectorization of the $i$-th action. Then, according to Lemma \ref{lem2}, we have $\|\Theta\|_{\mathcal{A}}=\inf_{\Theta} \left\{ \lambda\|\Theta\|_*+a_\mu \alpha \mathrm{tr}\left\{ \Theta^{\top} \widetilde{X}^\top L \widetilde{X}  \Theta \right\} \right\} $. 
\Halmos
\end{proof}

    

\section{Proof of the Theorem \ref{the1}}\label{appendixb}

\renewcommand{\thelemma}{B.\arabic{lemma}}  

\setcounter{lemma}{0}    

In this section, we will provide a detailed proof of Theorem \ref{the1}. The basic framework of the proof is as follows.
\begin{itemize}
    \item Step 1: Drawing on the idea of \cite{rao2015collaborative}, we equivalently transform the sum of the two regularizations into an atomic norm, leading to Lemma \ref{lem1}.
    \item Step 2: Building on the first step, we adopt the proof framework proposed by \cite{kang2022efficient}, handling the regularization term and the loss function separately, which correspond to Lemma \ref{lem22} and Lemma \ref{lem3}, respectively. Finally, we derive the bound for the parameter estimation error.
    \item Step 3: To demonstrate the improvement of the bound, we need to verify that the additional factor related to graph information in the bound is less than 1, as shown in Lemma \ref{lem6}.
\end{itemize}


First, we introduce the relevant lemmas, with their detailed proofs provided in subsection \ref{appendixb.1}-\ref{appendixb.3}.

\begin{lemma} For $\Theta^* =  AM B^\top$,
where the definitions of $A$ and $B$ can be found in Lemma \ref{lem1}, the low-rank matrix $M\overset{SVD}{=}(U,U_{\perp}) \Sigma^* (V,V_{\perp})^\top$, let $\Theta=\widehat{\Theta}-\mu^*\Theta^*$,
\begin{equation}
    \Lambda_1=\left(\begin{array}{cc}
0 & 0 \\
0 & U_{\perp}^{\top} A^{-1}\Theta B^{-\top} V_{\perp}
\end{array}\right), \quad \Lambda_2=\left(\begin{array}{cc}
U^{\top} A^{-1}\Theta B^{-\top} V & U^{\top} A^{-1}\Theta B^{-\top} V_{\perp} \\
U_{\perp}^{\top} A^{-1}\Theta B^{-\top} V & 0
\end{array}\right),
\end{equation} then 
\begin{itemize}
    \item $\left\|\mu^*\Theta^*\right\|_{\mathcal{A}}-\|\widehat{\Theta}\|_{\mathcal{A}} \leq \left\|A \left(U, U_{\perp}\right) \Lambda_2\left(V, V_{\perp}\right)^{\top} B^\top\right\|_{\mathcal{A}}-\left\|A \left(U, U_{\perp}\right) \Lambda_1\left(V, V_{\perp}\right)^{\top} B^\top\right\|_{\mathcal{A}}$;
    \item $\left\|\widehat{\Theta}-\mu^*\Theta^*\right\|_{\mathcal{A}}=\left\|A \left(U, U_{\perp}\right) \Lambda_1 \left(V, V_{\perp}\right)^{\top} B^\top\right\|_{\mathcal{A}}+\left\|A \left(U, U_{\perp}\right) \Lambda_2 \left(V, V_{\perp}\right)^{\top} B^\top\right\|_{\mathcal{A}}$.
\end{itemize}

    \label{lem22}
\end{lemma}

\begin{lemma}
    For the loss function defined as \( L_{T_1}(\Theta)\triangleq \langle\Theta, \Theta\rangle-\frac{2}{T_1} \sum_{i=1}^{T_1}\left\langle \psi_\nu\left(y_i \cdot S\left(X_i\right)\right), \Theta\right\rangle \), then use $$\beta=\frac{4\|A\| \|B\|}{\sqrt{T_1}} \sqrt{2 d_1 d_2 \gamma \left(4 \omega^2+r_{\max}^2\right) \log \left(\frac{2\left(d_1+d_2\right)}{\delta}\right)}, \nu=\sqrt{\frac{2 \log \left(\frac{2\left(d_1+d_2\right)}{\delta}\right)}{T_1 d_1 d_2 \gamma \left(4 \omega^2+r_{\max}^2\right)}},$$
    with probability at least $1-\delta $, we have
$\beta \geq 2 \|\nabla L\left(\mu^*\Theta^*\right)\|_{\mathcal{A}}^*$, where $\|\cdot\|_{\mathcal{A}}^*$ is the dual norm of the weighted atomic norm $\|\cdot\|_{\mathcal{A}}$.
\label{lem3}
\end{lemma}

\begin{lemma}
    Let $\lambda \leq 1$, $\alpha \leq \frac{1 -\lambda}{4 a_\mu n (n-1)}$, then $\|A^{-1}\|\leq 1,\|B^{-1}\|\leq 1$.
    \label{lem6}
\end{lemma} 

Next, we will complete the proof of Theorem \ref{the1} by following the steps summarized above and using the corresponding lemmas.

\begin{proof}{Proof of Theorem \ref{the1}}
According to the optimal problem and Taylor's expansion, we have
\begin{equation*}
    \left\{
             \begin{array}{lc}
             L(\hat{\Theta})+\beta\|\hat{\Theta}\|_{\mathcal{A}}\leq L(\mu^* \Theta^*)+\beta\|\mu^*\Theta^*\|_{\mathcal{A}}, &  \quad \text{optimal problem}\\
             L(\hat{\Theta})=L(\mu^* \Theta^*)+\langle \nabla L\left(\mu^* \Theta^*\right),\hat{\Theta}-\mu^*\Theta^*\rangle+2\|\hat{\Theta}-\mu^* \Theta^*\|_F^2. & \quad \text{Taylor's expansion} 
             \end{array}
\right.
\end{equation*}
Then, by swapping the order, we obtain 
\begin{equation}
   \|\hat{\Theta}-\mu^*\Theta^*\|_F^2 \leq -\frac{1}{2}\langle \nabla L\left(\mu^*\Theta^*\right),\hat{\Theta}-\mu^*\Theta^*\rangle+\frac{\beta}{2} \left[ \|\mu^* \Theta^*\|_{\mathcal{A}}-\|\hat{\Theta}\|_{\mathcal{A}}\right]. 
   \label{eq15}
\end{equation}
Next, by Cauchy–Schwarz inequality, \eqref{eq15} can be further expressed as 
\begin{equation}
    \|\hat{\Theta}-\mu^*\Theta^*\|_F^2\leq \frac{1}{2}\| \nabla L\left(\mu^* \Theta^*\right)\|_{\mathcal{A}}^*\|\hat{\Theta}-\mu^* \Theta^*\|_{\mathcal{A}}+\frac{\beta}{2} \left[ \|\mu^*\Theta^*\|_{\mathcal{A}}-\|\hat{\Theta}\|_{\mathcal{A}}\right].
\end{equation}
Due to Lemma \ref{lem22}, we can get 
\begin{equation}
\begin{aligned}
   \|\hat{\Theta}-\mu^*\Theta^*\|_F^2 & \leq \frac{1}{2}\left[\| \nabla L\left(\mu^*\Theta^*\right)\|_{\mathcal{A}}^* +\beta \right]\left\|A \left(U, U_{\perp}\right) \Lambda_2 \left(V, V_{\perp}\right)^{\top} B^\top\right\|_{\mathcal{A}} \notag \\&+\frac{1}{2} \left[ \| \nabla L \left(\mu^*\Theta^*\right)\|_{\mathcal{A}}^* -\beta \right] \left\|A \left(U, U_{\perp}\right) \Lambda_1 \left(V, V_{\perp}\right)^{\top} B^\top\right\|_{\mathcal{A}}. 
\end{aligned}
\end{equation}
Due to Lemma \ref{lem3}, we can get
\begin{equation}
    \begin{aligned}
        \|\hat{\Theta}-\mu^*\Theta^*\|_F^2&\leq \frac{3}{4}\beta \left\|A \left(U, U_{\perp}\right) \Lambda_2 \left(V, V_{\perp}\right)^{\top} B^\top\right\|_{\mathcal{A}}-\frac{1}{4} \beta \left\|A \left(U, U_{\perp}\right) \Lambda_1 \left(V, V_{\perp}\right)^{\top} B^\top\right\|_{\mathcal{A}}  \\
     &\leq \frac{3}{4}\beta \left\|A \left(U, U_{\perp}\right) \Lambda_2 \left(V, V_{\perp}\right)^{\top} B^\top\right\|_{\mathcal{A}}\\
    &\leq \frac{3}{4}\beta  \left\|\Lambda_2 \right\|_{*}\leq \frac{3}{4}\beta  \sqrt{2r}\left\|\Lambda_2\right\|_{F}\leq \frac{3}{4}\beta  \sqrt{2r}\left\|\Lambda\right\|_{F}\\
    &\leq \frac{3}{4}\beta  \sqrt{2r}\left\|A^{-1}(\widehat{\Theta}-\mu^*\Theta^*)B^{-\top}) \right\|_{F} \\
    &\leq \frac{3}{4}\beta \sqrt{2r}\left\|A^{-1}\right\| \left\|\widehat{\Theta}-\mu^* \Theta^*\right\|_{F} \left\|B^{-\top}\right\| \\
    &\leq \frac{3}{4}\beta  \sqrt{2r} \left\|\widehat{\Theta}-\mu^*\Theta^*\right\|_{F},   
    \end{aligned}
\end{equation}
that is $\left\|\widehat{\Theta}-\mu^* \Theta^*\right\|_{F}\leq \frac{3}{4}\beta \sqrt{2r} = \frac{6 \sqrt{\gamma(4\omega^2+r_{\max}^2)} \|A\| \|B\| \sqrt{d_1 d_2 r \log \left(\frac{2\left(d_1+d_2\right)} {\delta}\right)}}{\sqrt{T_1}}$. Due to Lemma \ref{lem6}, we have $\|A\|\leq \|A^{-1}\|\leq 1$ and $\|B\|\leq \|B^{-1}\|\leq 1$. Then letting $\zeta \triangleq \|A\| \|B\|$, we can obtain $\zeta \in (0,1]$.
\Halmos

\end{proof}


\subsection{Proof of the Lemma \ref{lem22}}\label{appendixb.1}


\begin{proof}{Proof of Theorem \ref{lem22}}
    
According to Definition \ref{def5}, we can deduce that $\Theta^* = \sum_i c_i A u_i v_i^\top B^\top \triangleq AM B^\top$, simultaneously satisfying $\text{rank}(M) = \text{rank}(\Theta^*) = r$. Therefore, we can conclude that $\Theta^* \in \mathcal{M} = \{A M B^\top: \text{rank}(M) = r, A = U_w S_w^{\frac{1}{2}}, B = U_h S_h^{\frac{1}{2}}\}$. Let $M=U \Sigma V^\top$ be the truncated SVD, where $U \in \mathbb{R}^{d_1 \times r}$, $V \in \mathbb{R}^{d_2 \times r}$. Then
\begin{equation}
    \Theta^*=A \left(U, U_{\perp}\right)\left(\begin{array}{cc}
\Sigma & 0 \\
0 & 0
\end{array}\right)\left(V, V_{\perp}\right)^{\top} B^\top=A \left(U, U_{\perp}\right) \Sigma^*\left(V, V_{\perp}\right)^{\top} B^\top,
\end{equation}
where $U_{\perp} \in \mathbb{R}^{d_1 \times\left(d_1-r\right)}, \Sigma^* \in \mathbb{R}^{d_1 \times d_2}$ and $V_{\perp} \in \mathbb{R}^{d_2 \times\left(d_2-r\right)}$. Furthermore, we define
$\Lambda=\left(U, U_{\perp}\right)^{\top} A^{-1}\Theta B^{-\top}\left(V, V_{\perp}\right)=\Lambda_1+\Lambda_2$, where let $\Theta=\widehat{\Theta}-\mu^*\Theta^*$,
\begin{equation}
    \Lambda_1=\left(\begin{array}{cc}
0 & 0 \\
0 & U_{\perp}^{\top} A^{-1}\Theta B^{-\top} V_{\perp}
\end{array}\right), \quad \Lambda_2=\left(\begin{array}{cc}
U^{\top} A^{-1}\Theta B^{-\top} V & U^{\top} A^{-1}\Theta B^{-\top} V_{\perp} \\
U_{\perp}^{\top} A^{-1}\Theta B^{-\top} V & 0
\end{array}\right),
\end{equation}
i.e. $\Theta=A \left(U, U_{\perp}\right) \Lambda \left(V, V_{\perp}\right)^{\top} B^\top$. Please note that we can express $\|Z\|_{\mathcal{A}}$ of Lemma \ref{lemA.1} as follows: 
\begin{equation}
    \begin{aligned}
\|Z\|_{\mathcal{A}}=&\inf_{W,H} \frac{1}{2} \left\{ \mathrm{tr}(W^\top L_w W)+\mathrm{tr}(H^\top L_h H) \right\}=\inf_{W,H} \frac{1}{2} \{\|A^{-1}W\|_F^2 +\|B^{-1}H\|_F^2 \}\\
=&\inf_{W,H} \left\{\|A^{-1}W (B^{-1}H)^\top\|_*\right\}=\|A^{-1}Z B^{-\top}\|_*.
    \end{aligned}
\end{equation}

On the one hand, 
for $\left\|\mu^*\Theta^*\right\|_{\mathcal{A}}-\|\widehat{\Theta}\|_{\mathcal{A}}$, it holds that
\refstepcounter{equation}
\begin{equation}
    \begin{aligned}
&\left\| \widehat{\Theta}\right\|_{\mathcal{A}}=\left\| \mu^*\Theta^*+\Theta\right\|_{\mathcal{A}}=\left\| A^{-1}(\mu^*\Theta^*+\Theta)B^{-\top}\right\|_{*}
=\left\| \mu^*\Sigma^*+\Lambda\right\|_{*} \\
=&\left\| \mu^*\Sigma^*+\Lambda_1+\Lambda_2\right\|_{*}  \geq \left\|\mu^*\Sigma^*+\Lambda_1\right\|_{*}-\left\|\Lambda_2\right\|_{*} 
=\left\|\mu^*\Sigma^*\right\|_{*}+\left\|\Lambda_1\right\|_{*}-\left\|\Lambda_2\right\|_{*} \\
=&\left\| A^{-1}(A \left(U, U_{\perp}\right)\mu^* \Sigma^*\left(V, V_{\perp}\right)^{\top} B^\top)B^{-\top}\right\|_{*}+\left\|A^{-1}(A \left(U, U_{\perp}\right) \Lambda_1\left(V, V_{\perp}\right)^{\top} B^\top)B^{-\top}\right\|_{*}\notag \\&-\left\|A^{-1}(A \left(U, U_{\perp}\right) \Lambda_2\left(V, V_{\perp}\right)^{\top} B^\top)B^{-\top}\right\|_{*}\\
=&\left\| A \left(U, U_{\perp}\right) \mu^*\Sigma^*\left(V, V_{\perp}\right)^{\top} B^\top\right\|_{\mathcal{A}}+\left\|A \left(U, U_{\perp}\right) \Lambda_1\left(V, V_{\perp}\right)^{\top} B^\top\right\|_{\mathcal{A}}-\left\|A \left(U, U_{\perp}\right) \Lambda_2\left(V, V_{\perp}\right)^{\top} B^\top\right\|_{\mathcal{A}}\\
=&\left\| \mu^*\Theta^*\right\|_{\mathcal{A}}+\left\|A \left(U, U_{\perp}\right) \Lambda_1\left(V, V_{\perp}\right)^{\top} B^\top\right\|_{\mathcal{A}}-\left\|A \left(U, U_{\perp}\right) \Lambda_2\left(V, V_{\perp}\right)^{\top} B^\top\right\|_{\mathcal{A}}, 
    \end{aligned}
\end{equation}
which implies that $\left\|\mu^*\Theta^*\right\|_{\mathcal{A}}-\|\widehat{\Theta}\|_{\mathcal{A}} \leq \left\|A \left(U, U_{\perp}\right) \Lambda_2\left(V, V_{\perp}\right)^{\top} B^\top\right\|_{\mathcal{A}}-\left\|A \left(U, U_{\perp}\right) \Lambda_1\left(V, V_{\perp}\right)^{\top} B^\top\right\|_{\mathcal{A}}$.

On the other hand, 
for $\left\|\widehat{\Theta}-\mu^*\Theta^*\right\|_{\mathcal{A}}$, we can obtain
\refstepcounter{equation}
\begin{equation}
    \begin{aligned}
        \left\|\widehat{\Theta}-\mu^*\Theta^*\right\|_{\mathcal{A}}=&\left\|\Theta\right\|_{\mathcal{A}}=\left\|A \left(U, U_{\perp}\right) \Lambda \left(V, V_{\perp}\right)^{\top} B^\top\right\|_{\mathcal{A}}
=\left\|A^{-1}(A \left(U, U_{\perp}\right) \Lambda \left(V, V_{\perp}\right)^{\top} B^\top)B^{-\top}\right\|_{*}\\
=&\left\|\Lambda_1+\Lambda_2\right\|_{*}
=\left\|\Lambda_1\right\|_{*}+\left\|\Lambda_2\right\|_{*}
=\left\|A^{-1}(A \left(U, U_{\perp}\right) \Lambda_1 \left(V, V_{\perp}\right)^{\top} B^\top)B^{-\top}\right\|_{*} \notag \\&+\left\|A^{-1}(A \left(U, U_{\perp}\right) \Lambda_2 \left(V, V_{\perp}\right)^{\top} B^\top)B^{-\top}\right\|_{*}\\
=&\left\|A \left(U, U_{\perp}\right) \Lambda_1 \left(V, V_{\perp}\right)^{\top} B^\top\right\|_{\mathcal{A}}+\left\|A \left(U, U_{\perp}\right) \Lambda_2 \left(V, V_{\perp}\right)^{\top} B^\top\right\|_{\mathcal{A}}.
    \end{aligned}
\end{equation}
\Halmos
\end{proof}

\subsection{Proof of the Lemma \ref{lem3}}\label{appendixb.2}

This subsection aims to prove the relationship between the norm of the gradient of the loss function (specifically, the dual norm of the regularization norm) and the adjusted parameter, as stated in the Lemma \ref{lem3}. Before formally presenting the proof of the Lemma \ref{lem3}, it is necessary to clarify the form of the dual norm of the atomic norm, as stated in Lemma \ref{lem4}.

\begin{lemma}[\cite{rao2015collaborative}]
  The dual norm of the weighted atomic norm in Definition \ref{def5} is $\|Z\|_{\mathcal{A}}^*=\left\|A^\top Z B\right\|$.
  \label{lem4}
\end{lemma}

In addition, despite the difference in the regularization norm between this paper and \cite{kang2022efficient}, the corresponding loss function is the same. Therefore, in the proof, we shall utilize a conclusion from that paper, which we present as follows.

\begin{lemma}[Lemma B.4 in \cite{kang2022efficient}]
    $L: \mathbb{R}^{d_1 \times d_2} \rightarrow \mathbb{R}$ is the loss function defined as \( L_{T_1}(\Theta)\triangleq \langle\Theta, \Theta\rangle-\frac{2}{T_1} \sum_{i=1}^{T_1}\left\langle \psi_\nu\left(y_i \cdot S\left(X_i\right)\right), \Theta\right\rangle \). Then by setting
    \begin{equation*}
        \begin{aligned}
           t=\sqrt{2 d_1 d_2 \gamma \left(4 \omega^2+r_{\max}^2\right) \log \left(\frac{2\left(d_1+d_2\right)}{\delta}\right)}, 
\nu=\frac{t}{\left(4 \omega^2+r_{\max
}^2\right) \gamma d_1 d_2 \sqrt{T_1}}=\sqrt{\frac{2 \log \left(\frac{2\left(d_1+d_2\right)}{\delta}\right)}{T_1 d_1 d_2 \gamma \left(4 \omega^2+r_{\max}^2\right)}},  
        \end{aligned}
    \end{equation*}
we have with probability at least $1-\delta$, it holds that $P\left(\left\|\nabla L\left( \mu^*\Theta^*\right)\right\| \geq \frac{2 t}{\sqrt{T_1}}\right) \leq \delta$,
where $\mu^*=\mathbb{E}[\mu^\prime (\langle X, \Theta^* \rangle)]$.
\label{lem5}
\end{lemma}

\begin{proof}{Proof of Lemma \ref{lem3}}
\begin{equation}
    \begin{aligned}
       &P\left(\left\|\nabla L\left( \mu^* \Theta^*\right)\right\|_{\mathcal{A}}^* \geq \frac{\beta}{2}\right)=P\left(\left\|A^\top \nabla L\left(\mu^* \Theta^*\right)B\right\| \geq \frac{\beta}{2}\right)\\
     \leq & P\left(\left\|A\right\| \left\| \nabla L\left( \mu^*\Theta^*\right)\right\| \left\|B\right\| \geq \frac{\beta}{2}\right)
    =P\left( \left\| \nabla L\left( \mu^*\Theta^*\right)\right\|  \geq \frac{\beta}{2\left\|A\right\| \left\|B\right\|}\right). 
    \end{aligned}
\end{equation}
Then according to Lemma \ref{lem5}, let $\beta=\frac{4\left\|A\right\| \left\|B\right\|}{\sqrt{T_1}} \sqrt{2 d_1 d_2 \gamma \left(4 \omega^2+r_{\max}^2\right) \log \left(\frac{2\left(d_1+d_2\right)}{\delta}\right)}$, we have completed the proof of Lemma \ref{lem3}.
\Halmos
\end{proof}

\subsection{Proof the Lemma \ref{lem6}}\label{appendixb.3}


The proof process of the Lemma \ref{lem6} will involve some fundamental facts about the singular values. To this end, we describe it as follows:
\begin{itemize}
    \item For any matrix $A,B$, $\sigma_i(A^\top B)\leq \sigma_i(A)\sigma_1(B)$, where $\sigma_i(\cdot)$ is the $i$-th largest singular value.
    \item For any matrix $A\in \mathbb{R}^{n \times n}$, $B\in \mathbb{R}^{p \times p}$, let $\lambda_1,\lambda_1,\dots,\lambda_n$
 be the singular values of $A$, $\mu_1,\mu_1,\dots,\mu_p$
 be the singular values of $B$. 
Then the singular values of $A \otimes B$ are $\lambda_i \mu_j, \; i=1,2,\dots,n, \; j=1,2,\dots,p$.
    \item For any Laplacian matrix $L$, $\sigma_{\min}(L)=0$.
    \item For the graph Laplacian $L = D - W$, $ \sigma_{\max}(L) \leq 2 d_{\max}\leq 2(n-1)$, where $d_{\max} = \max_i \sum_j W_{ij}$ is the maximum node degree, $n$ is the number of nodes. 
\end{itemize}

\begin{proof}{Proof of Lemma \ref{lem6}}
According to the definition of $B$ in Lemma \ref{lem1}, we have $\|B^{-1}\|=\|(U_h S_h^{-\frac{1}{2}})^{-1}\| =\| S_h^{\frac{1}{2}}\| =\sigma_{\max}\left(S_h^{\frac{1}{2}}\right)  =\sigma_{\max}^{\frac{1}{2}} \left(S_h\right) =\sigma_{\max}^{\frac{1}{2}} \left(L_h\right)  =\sigma_{\max}^{\frac{1}{2}} \left(\lambda I\right)  = \lambda^{\frac{1}{2}} $. Let $\lambda \leq 1$, then $\|B^{-1}\| \leq1$. Similarly, we can obtain $\|A^{-1}\|=\|(U_w S_w^{-\frac{1}{2}})^{-1}\| =\| S_w^{\frac{1}{2}} U_w\| =\| S_w^{\frac{1}{2}}\| =\sigma_{\max}\left(S_w^{\frac{1}{2}}\right)  =\sigma_{\max}^{\frac{1}{2}} \left(S_w\right)  =\sigma_{\max}^{\frac{1}{2}} \left(L_w\right)$. Next, we further bound $\sigma_{\max} \left(L_w\right)$ by
\begin{equation}
    \begin{aligned} \sigma_{\max}\left(L_w\right)&=\sigma_{\max}\left(I \otimes L_w\right)=\sigma_{\max}\left(2a_\mu \alpha\left(H^\top \otimes I \right)\widetilde{X}^\top L \widetilde{X} \left(H \otimes I \right)+\lambda I\right)\\
&=2a_\mu \alpha \sigma_{\max}\left(\left(H^\top \otimes I \right)\widetilde{X}^\top L \widetilde{X} \left(H \otimes I \right)\right)+\lambda \\
&\leq 2a_\mu \alpha \sigma_{\max}\left(H^\top \otimes I \right)\sigma_{\max}\left(\widetilde{X}^\top L \widetilde{X} \left(H \otimes I \right)\right)+\lambda \\
&\leq 2a_\mu \alpha \sigma_{\max}\left(H^\top \otimes I \right)\sigma_{\max}\left(\widetilde{X}^\top L \widetilde{X}\right)\sigma_{\max}\left(H\otimes I \right)+\lambda \\
&\leq 2a_\mu \alpha \sigma_{\max}\left(H^\top \otimes I \right)\sigma_{\max}\left(\widetilde{X}^\top \right)\sigma_{\max}\left( L \right)\sigma_{\max}\left(\widetilde{X} \right)\sigma_{\max}\left(H\otimes I \right)+\lambda \\
&=2a_\mu \alpha \sigma_{\max}^2\left(H \right) \sigma_{\max}^2\left(\widetilde{X} \right) \sigma_{\max}\left(L \right)+\lambda\\
&=2a_\mu \alpha \sigma_{\max}^2\left(V_\Theta S_\Theta^{\frac{1}{2}} \right) \sigma_{\max}^2\left(\widetilde{X} \right) \sigma_{\max}\left(L \right)+\lambda\\
&=2a_\mu \alpha \sigma_{\max}\left(\Theta  \right) \sigma_{\max}^2\left(\widetilde{X} \right) \sigma_{\max}\left(L \right)+\lambda\\
&\leq2a_\mu \alpha \left\|\Theta  \right\|_F \sigma_{\max}^2\left(\widetilde{X} \right) \sigma_{\max}\left(L \right)+\lambda\\
&\leq2a_\mu \alpha \sigma_{\max}^2\left(\widetilde{X} \right) \sigma_{\max}\left(L \right)+\lambda,
\end{aligned}
\end{equation}
where the condition for the last inequality to hold is the assumption $\|\Theta\|_F \leq 1$ made without loss of generality.
Then, letting 
\begin{equation}
 \alpha  \leq \frac{1-\lambda}{2 a_\mu  \sigma_{\max}^2\left(\widetilde{X} \right) \sigma_{\max}\left(L \right)},
 \label{alpha}
\end{equation}
we have $\|A^{-1}\| =\sigma_{\max}^{\frac{1}{2}} \left(L_w\right) \leq1$. For practical purposes, using the bounds $\sigma_{\max}^2\left(\widetilde{X} \right) \leq \left\|\widetilde{X} \right\|_F^2 \leq n, \sigma_{\max}\left(L \right)\leq 2(n-1)$, we further relax the condition to $\alpha \leq \frac{1 -\lambda}{4 a_\mu n (n-1)}$, to avoid the explicit computation of singular values.

    \Halmos
\end{proof}

\section{Proof of the Theorem \ref{the2}}

\renewcommand{\thelemma}{C.\arabic{lemma}}  

\setcounter{lemma}{0}    

Theorem \ref{the2} delineates the confidence interval of expected reward. In its proof, we primarily rely on the Fundamental Theorem of Calculus and inequalities involving weighted norm, thereby breaking it down into two components involving noise and true parameter. The proofs of these two components are associated with the following two lemmas, respectively.

\begin{lemma}[Lemma C.3 in \cite{kang2022efficient}]
  Let $\left\{F_t\right\}_{t=0}^{\infty}$ be a filtration and $\left\{\eta_t\right\}_{t=1}^{\infty}$ be a real-valued stochastic process such that $\eta_t$ is $F_t$-measurable and $\eta_t$ is conditionally $\omega$-sub-Gaussian for some $\omega \geq 0$ i.e.
$$
\forall \lambda \in \mathbb{R}, \; \mathbb{E}\left[e^{\lambda \eta_t} \mid F_{t-1}\right] \leq \exp \left(\frac{\lambda^2 \omega^2}{2}\right) .
$$
Let $\left\{\boldsymbol{x}_t\right\}_{t=1}^{\infty}$ be an $\mathbb{R}^d$-valued stochastic process such that $\boldsymbol{x}_t$ is $F_{t-1}$-measurable. Assume that $V$ is a $d \times d$ positive definite matrix and independent with sample random variables after time $m$. For any $t \geq 2$, define
$$
\bar{V}_t=V+\sum_{s=1}^t \boldsymbol{x}_s \boldsymbol{x}_s^{\top} \quad \boldsymbol{s}_t=\sum_{s=1}^t \eta_s \boldsymbol{x}_s .
$$
Then, for any $\delta>0$, with probability at least $1-\delta$, for all $t \geq m+1$,
$$
\left\|\boldsymbol{s}_t\right\|_{\bar{V}_t^{-1}}^2 \leq  \omega^2 \log \left(\frac{\left|\bar{V}_t\right| }{|V| \delta^2}\right)
.$$  
\label{lem7}
\end{lemma}


\begin{lemma}[Lemma 16 in \cite{kocak2020spectral}]
    For any symmetric, positive semi-definite matrix $X$, and any vectors $\boldsymbol{u}$ and $\boldsymbol{y}$, $\boldsymbol{y}^{\top}\left(X+\boldsymbol{u u}^{\top}\right)^{-1} \boldsymbol{y} \leq \boldsymbol{y}^{\top} X^{-1} \boldsymbol{y}$.
\label{lem8}
\end{lemma}

\begin{proof}{Proof of Theorem \ref{the2}}
First, we appropriately scale \(\left|\mu\left(\boldsymbol{x}^{\top} \boldsymbol{\theta}^*\right)-\mu\left(\boldsymbol{x}^{\top} \hat{\boldsymbol{\theta}}_t\right)\right|\). According to the Fundamental Theorem of Calculus, we have $\left|\mu\left(\boldsymbol{x}^{\top} \boldsymbol{\theta}^*\right)-\mu\left(\boldsymbol{x}^{\top} \hat{\boldsymbol{\theta}}_t\right)\right| \leq k_\mu \left|\boldsymbol{x}^{\top}\left(\theta^*-\hat{\theta}_t\right)\right|$, i.e. $\mu(\cdot)$ is $k_\mu$-Lipshitz continuous. Applying this theorem again, let $G_t=\int_0^1  g_t^\prime \left(s \boldsymbol{\theta}^*+(1-s) \hat{\boldsymbol{\theta}}_t\right) d s$, where $g_t(\boldsymbol{\theta})=\sum_{i=1}^{T_1} \mu(\boldsymbol{x}_{h, i}{ }^{\top} \boldsymbol{\theta})  \boldsymbol{x}_{h, i}+\sum_{i=1}^{t-1} \mu (\boldsymbol{x}_i^{\top} \boldsymbol{\theta}) \boldsymbol{x}_i+\Lambda \boldsymbol{\theta}+a_\mu \alpha \widetilde{X}^\top L \widetilde{X} \boldsymbol{\theta}$, we obtain $\left|\boldsymbol{x}^{\top}\left(\boldsymbol{\theta}^*-\hat{\boldsymbol{\theta}}_t\right)\right|= \left|\boldsymbol{x}^{\top} G_t^{-1}\left(g_t\left(\boldsymbol{\theta}^*\right)-g_t\left(\hat{\boldsymbol{\theta}}_t\right)\right)\right|$. Due to the Cauchy-Schwarz inequality, we can obtain
\begin{equation}
    \left|\boldsymbol{x}^{\top} G_t^{-1}\left(g_t\left(\boldsymbol{\theta}^*\right)-g_t\left(\hat{\boldsymbol{\theta}}_t\right)\right)\right| \leq \|\boldsymbol{x}\|_{G_t^{-1}} \|G_t^{-1}\left(g_t\left(\boldsymbol{\theta}^*\right)-g_t\left(\hat{\boldsymbol{\theta}}_t\right)\right)\|_{G_t}=\|\boldsymbol{x}\|_{G_t^{-1}} \|g_t\left(\boldsymbol{\theta}^*\right)-g_t\left(\hat{\boldsymbol{\theta}}_t\right)\|_{G_t^{-1}}.
    \label{eq.c.1}
\end{equation}
Combining the above inequalities, we obtain $\left|\mu\left(\boldsymbol{x}^{\top} \boldsymbol{\theta}^*\right)-\mu\left(\boldsymbol{x}^{\top} \hat{\boldsymbol{\theta}}_t\right)\right| \leq k_\mu \|\boldsymbol{x}\|_{G_t^{-1}} \|g_t\left(\boldsymbol{\theta}^*\right)-g_t\left(\hat{\boldsymbol{\theta}}_t\right)\|_{G_t^{-1}}$. 

According to $g_t^\prime(\boldsymbol{\theta})=\sum_{i=1}^{T_1} \mu^{\prime}\left(\boldsymbol{x}_{h, i}^{\top} \boldsymbol{\theta}\right) \boldsymbol{x}_{h, i} \boldsymbol{x}_{h, i}^{\prime}+\sum_{k=1}^{t-1} \mu^{\prime}\left(\boldsymbol{x}_k^{\top} \boldsymbol{\theta}\right) \boldsymbol{x}_k \boldsymbol{x}_k^{\top}+\Lambda+a_\mu \alpha \widetilde{X}^\top L \widetilde{X} \succeq c_\mu V_t\left(c_\mu\right)$, then we have $G_t \succeq c_\mu V_t\left(c_\mu\right)$, that is $G_t^{-1} \preceq \frac{V_t^{-1}\left(c_\mu\right)}{c_\mu}$. In addition, based on the definition of $g_t(\cdot)$, we have $g_t\left(\boldsymbol{\theta}^*\right)-g_t\left(\hat{\boldsymbol{\theta}}_t\right)=-\sum_{k=1}^{T_1} \epsilon_{h, k} \boldsymbol{x}_{h, k}-\sum_{k=1}^{t-1} \epsilon_k \boldsymbol{x}_k-\Lambda \boldsymbol{\theta}^*-a_\mu \alpha \widetilde{X}^\top L \widetilde{X} \boldsymbol{\theta}^*$. Based on the above processing, we have 
\begin{equation}
\begin{aligned}
    &\left|\mu\left(\boldsymbol{x}^{\top} \boldsymbol{\theta}^*\right)-\mu\left(\boldsymbol{x}^{\top} \hat{\boldsymbol{\theta}}_t\right)\right| \leq k_\mu \|\boldsymbol{x}\|_{G_t^{-1}} \|g_t\left(\boldsymbol{\theta}^*\right)-g_t\left(\hat{\boldsymbol{\theta}}_t\right)\|_{G_t^{-1}}\\ 
    \leq & \frac{k_\mu}{c_\mu}\|\boldsymbol{x}\|_{V_t^{-1}\left(c_\mu\right)} \left[ \left\|\sum_{k=1}^{T_1} \epsilon_{h, k} \boldsymbol{x}_{h, k}+\sum_{k=1}^{t-1} \epsilon_k \boldsymbol{x}_k\right\|_{V_t^{-1}\left(c_\mu\right)}+\left\|\Lambda \boldsymbol{\theta}^*+a_\mu \alpha \widetilde{X}^\top L \widetilde{X} \boldsymbol{\theta}^*\right\|_{V_t^{-1}\left(c_\mu\right)}\right].
\end{aligned}
\label{eq.c.2}
\end{equation}
    
Next, we will separately handle the two terms inside the brackets in the above inequality. Specifically for the term $\left\|\sum_{k=1}^{T_1} \epsilon_{h, k} \boldsymbol{x}_{h, k}+\sum_{k=1}^{t-1} \epsilon_k \boldsymbol{x}_k\right\|_{V_t^{-1}\left(c_\mu\right)}$, directly applying Lemma \ref{lem7} yields
\begin{equation}
 \left\|\sum_{k=1}^{T_1} \epsilon_{h, k} \boldsymbol{x}_{h, k}+\sum_{k=1}^{t-1} \epsilon_k \boldsymbol{x}_k\right\|_{V_t^{-1}\left(c_\mu\right)} \leq \omega \sqrt{\log \left(\frac{\left|V_t\left(c_\mu\right)\right|}{\left|V(c_\mu)\right|\delta^2}\right)}.
 \label{eq.c.3}
\end{equation}

For the term $\left\| \left( \Lambda+a_\mu \alpha \widetilde{X}^\top L\widetilde{X}\right) \boldsymbol{\theta}^*\right\|_{V_t^{-1}\left(c_\mu\right)}$, we have
\begin{equation}
    \begin{aligned}
    &\left\| \left( \Lambda+a_\mu \alpha \widetilde{X}^\top L\widetilde{X}\right)\boldsymbol{\theta}^*\right\|_{V_t^{-1}\left(c_\mu\right)}=\sqrt{\boldsymbol{\theta}^{* \top} 
 \left( \Lambda+a_\mu \alpha \widetilde{X}^\top L\widetilde{X}\right)V_t^{-1}\left(c_\mu\right)\left( \Lambda+a_\mu \alpha \widetilde{X}^\top L\widetilde{X}\right)\boldsymbol{\theta}^*}\\
 \leq & \sqrt{c_\mu \boldsymbol{\theta}^{* \top} 
 \left( \Lambda+a_\mu \alpha \widetilde{X}^\top L\widetilde{X}\right)\boldsymbol{\theta}^*}=\sqrt{c_\mu \left( \boldsymbol{\theta}^{*\top} 
 \Lambda \boldsymbol{\theta}^*+\boldsymbol{\theta}^{*\top}  a_\mu \alpha \widetilde{X}^\top L \widetilde{X} \boldsymbol{\theta}^*\right)} \\
 =& \sqrt{c_\mu \left(\lambda_2 \sum_{i=1}^k \boldsymbol{\theta}_i^{*2}+ \lambda_\perp \sum_{i=k+1}^{d_1 d_2} \boldsymbol{\theta}_i^{*2}+\boldsymbol{\theta}^{* \top} a_\mu \alpha \widetilde{X}^\top L \widetilde{X} \boldsymbol{\theta}^* \right) }\\
 \leq & \sqrt{c_\mu \left( \lambda_2 +\lambda_\perp \tau^{2}+ a_\mu \alpha \sigma_{\max}(\widetilde{X}^\top L \widetilde{X}) \right)} \leq \sqrt{c_\mu} \left[ \sqrt{\lambda_2} +\sqrt{\lambda_\perp} \tau+ \sqrt{a_\mu \alpha \sigma_{\max}(\widetilde{X}^\top L \widetilde{X})} \right]\\
 \leq & \sqrt{c_\mu} \left( \sqrt{\lambda_2} +\sqrt{\lambda_\perp} \tau+ \sqrt{\frac{a_\mu (1-\lambda) \sigma_{\max}(\widetilde{X}^\top L \widetilde{X})}{2 a_\mu \sigma_{\max}^{2}\left(\widetilde{X}\right)\sigma_{\max}(L)} } \right)\\
 \leq & \sqrt{c_\mu} \left( \sqrt{\lambda_2} +\sqrt{\lambda_\perp} \tau+ 1 \right), 
\end{aligned}
\label{eq.c.4}
\end{equation}
where the first inequality holds is based on Lemma \ref{lem8}, the fourth inequality is based on \eqref{alpha}.

Finally, substituting \eqref{eq.c.3} and \eqref{eq.c.4} into \eqref{eq.c.2} yields the conclusion of the Theorem \ref{the2}.
\Halmos
\end{proof}

\section{Proof of the Theorem \ref{the4}}
\renewcommand{\thelemma}{D.\arabic{lemma}}  

\setcounter{lemma}{0}    



Theorem~\ref{the4} provides the regret bound for the arm selection strategy, and its proof consists of two main steps. First, based on the confidence bound established in Theorem \ref{the2} and using Lemmas \ref{lem11} and \ref{lem12}, we derive a regret that includes the term $\log \left( \frac{ |V_t(c_\mu)| }{ |V(c_\mu)| } \right)$. Second, when further bounding this term, we observe that the matrix $V(c_\mu)$ is no longer diagonal due to the incorporation of graph information, making Lemma 2 in \cite{jun2019bilinear} inapplicable. To address this, we extend the result to the general case of positive definite matrix (as shown in Lemma \ref{lem.E1}), and by carefully choosing the parameter $\lambda_\perp$, we complete the derivation of the final regret bound.

\begin{lemma}[Lemma 3 in \cite{jun2017scalable}]
   \begin{equation*}
       \min\{a, x\} \leq \max\{2, a\} \log(1 + x), \; \forall a, x > 0.
   \end{equation*}
    
\label{lem11}
\end{lemma}

\begin{lemma}[Lemma 11 in \cite{abbasi2011improved}]
    Let $\{\boldsymbol{x}_t\}_{t=1}^\infty
 $be a sequence in $\mathbb{R
}^d$, $V$ is the $d \times d$ positive definite matrix and define $V_t = V +\sum_{s=1}^t \boldsymbol{x}_s \boldsymbol{x}_s^\top $. Then, we have that
$$\sum_{t=1}^T \log (1+\|\boldsymbol{x}_t\|_{V_{t-1}^{-1}}^2) = \log \frac{|V_T|}{|V|}.$$
\label{lem12}
\end{lemma}

\begin{lemma}[The Bound about Matrix Determinant]
For any $t$, let $V_t \triangleq \sum_{s=1}^{t-1} \boldsymbol{x}_s \boldsymbol{x}_s^{\top}+V$, where $V$ is a positive definite matrix. Then,
\begin{equation*}
    \log \frac{\left|V_t\right|}{|V|} \leq \max_{t_i: \sum_{i=1}^{d_1 d_2} t_i=t-1} \sum_{i=1}^{d_1 d_2} \log \left(1+\frac{t_i}{\sigma_i(V)}\right).
\end{equation*}
\label{lem.E1}
\end{lemma}

\begin{proof}{Proof of Theorem \ref{the4}}
According to the Cauchy-Schwarz inequality and the relationship among norms, let $e_T=\omega \sqrt{\log \frac{\left|V_T(c_\mu)\right|}{|V(c_\mu)| \delta^2}}+\sqrt{c_\mu} \left(\sqrt{\lambda_2} +\sqrt{\lambda_{\perp}} \tau+1\right)$, we can obtain $$r_t \triangleq \mu \left(\langle \boldsymbol{x}^*, \boldsymbol{\theta}^* \rangle\right) - \mu \left(\langle \boldsymbol{x}_t, \boldsymbol{\theta}^* \rangle\right)
    \leq \min \left\{ 2k_\mu,2 \frac{k_\mu}{c_\mu} e_t \|\boldsymbol{x}_t\|_{V_{t}^{-1}(c_\mu)} \right\}
     \leq 2 \frac{k_\mu}{c_\mu} e_T \min \left\{ \frac{c_\mu}{ e_T}, \|\boldsymbol{x}_t\|_{V_{t}^{-1}(c_\mu)} \right\}.$$ By applying the Cauchy-Schwarz inequality again, we can obtain
\begin{equation}
    \sum_{t=1}^T r_t\leq \sqrt{T \sum_{t=1}^T r_t^2}\leq 2 \frac{k_\mu}{c_\mu} e_T \sqrt{T \sum_{t=1}^T \min \left\{ \frac{c_\mu^2}{ e_T^2}, \|\boldsymbol{x}_t\|_{V_{t}^{-1}(c_\mu)}^2 \right\}}.
\end{equation}
Based on Lemma \ref{lem11}, we can obtain
\begin{equation}
    \sum_{t=1}^T r_t \leq 2 \frac{k_\mu}{c_\mu} e_T \sqrt{T  \max \left\{ 2,\frac{c_\mu^2}{ e_T^2} \right\}\sum_{t=1}^T \log \left(1+\|\boldsymbol{x}_t\|_{V_{t}^{-1}(c_\mu)}^2\right) }.
\end{equation}
Furthermore, due to Lemma \ref{lem12}, we have 
\begin{equation}
\begin{aligned}
    \sum_{t=1}^T r_t &\leq 2 \frac{k_\mu}{c_\mu} e_T \sqrt{T  \max \left\{ 2,\frac{c_\mu^2}{ e_T^2} \right\} \log \frac{|V_T(c_\mu)|}{|V(c_\mu)+\sum_{i=1}^{T_1} \boldsymbol{x}_{h, i} \boldsymbol{x}_{h, i}^{\top}|}} \\
     &\leq 2 \frac{k_\mu}{c_\mu} e_T \sqrt{T  \max \left\{ 2,\frac{c_\mu^2}{ e_T^2} \right\} \log \frac{|V_T(c_\mu)|}{|V(c_\mu)|}} .
\end{aligned}
\label{eq.d.3}
\end{equation}
Next, we perform the $e_T^2$ as
\begin{equation}
    \begin{aligned}
        e_T^2&=\left(\omega \sqrt{ \log \frac{\left|V_T(c_\mu)\right|}{|V(c_\mu)| \delta^2}}+\sqrt{c_\mu}\left(\sqrt{\lambda_2} +\sqrt{\lambda_{\perp}} \tau+1\right)\right)^2\\
    &\geq c_\mu \left(\sqrt{\lambda_2} +\sqrt{\lambda_{\perp}} \tau+1\right)^2\geq c_\mu \left(\lambda_2 +1\right).
    \end{aligned}
    \label{eq.d.4}
\end{equation}
That is to say, $\frac{c_\mu^2}{ e_T^2} \leq \frac{c_\mu}{\lambda_2+1}\leq \frac{c_\mu}{\lambda_2}$. By substituting \eqref{eq.d.4} into \eqref{eq.d.3} yields
\begin{equation}
R_T\triangleq \sum_{t=1}^T r_t=\mathcal{O}\left(e_T\sqrt{T  \log \frac{\left|V_T(c_\mu)\right|}{|V(c_\mu)|}}  \right).
  \label{eq.d.5}  
\end{equation}

According to the Lemma \ref{lem.E1}, we have $\log\frac{\left|V_T(c_\mu)\right|}{|V(c_\mu)|} \leq \max_{t_i: \sum_{i=1}^{d_1 d_2} t_i=T-1} \sum_{i=1}^{d_1 d_2} \log \left(1+\frac{t_i}{\sigma_i(V(c_\mu))}\right)$.
Considering that $\sigma_i(V(c_\mu))= \sigma_i\left(\frac{\Lambda+a_\mu \alpha \widetilde{X}^\top L\widetilde{X}}{c_\mu}\right)\geq \frac{1}{c_\mu} \left(\sigma_i(\Lambda)+\sigma_{\min}(a_\mu \alpha \widetilde{X}^\top L\widetilde{X})\right)\geq \frac{1}{c_\mu} \sigma_i\left(\Lambda\right)$, we can obtain
\begin{equation}
\begin{aligned}
    \log \frac{\left|V_T(c_\mu)\right|}{|V(c_\mu)|} & \leq \max \sum_{i=1}^{d_1 d_2} \log \left(1+\frac{c_\mu t_i}{\sigma_i\left(\Lambda\right)}\right)  \leq  \max \sum_{i=1}^{k} \log \left(1+\frac{c_\mu t_i}{\lambda_2}\right)+ \max \sum_{i=k+1}^{d_1 d_2} \log \left(1+\frac{c_\mu t_i}{\lambda_{\perp}}\right)\\
    & \leq  k \log \left(1+\frac{c_\mu T}{\lambda_2}\right)+  \sum_{i=k+1}^{d_1 d_2} \frac{c_\mu t_i}{\lambda_{\perp}}\leq k \log \left(1+\frac{c_\mu T}{\lambda_2}\right)+   \frac{c_\mu T}{\lambda_{\perp}}.
\end{aligned}
\end{equation}
If $\lambda_{\perp}=\frac{c_\mu T}{k \log (1+c_\mu T/\lambda_2)}$, then $\log \frac{\left|V_T(c_\mu)\right|}{|V(c_\mu)|} \leq 2k \log \left( 1+\frac{c_\mu T}{\lambda_2}\right)$. Substituting it into \eqref{eq.d.5}, the regret bound is
$\tilde{\mathcal{O}}\left( \omega k \sqrt{T} +  \tau T\right)$.
\Halmos

\end{proof}

\subsection{Proof of the Lemma \ref{lem.E1}}




Lemma \ref{lem.E1} deals with the bound on the determinant ratio of matrices. Due to the incorporation of graph information, the original diagonal matrix in the ratio transforms into a positive definite matrix, rendering the Lemma 2 in \cite{jun2019bilinear} no longer applicable. Therefore, we extend this approach to Lemma \ref{lem10}, which is primarily based on the fact elucidated in Lemma \ref{lem9}. 


\begin{lemma}[Lemma 21 in \cite{kocak2020spectral}]
    For any real positive-definite matrix $A$ with only simple eigenvalue multiplicities and any vector $\boldsymbol{x}$ such that $\|\boldsymbol{x}\| \leq 1$, we have that the determinant $\left|A+\boldsymbol{x x}^{\top}\right|$ is maximized by a vector $\boldsymbol{x}$ which is aligned with an eigenvector of $A$.
    \label{lem9}
\end{lemma}

\begin{lemma}[The generalization of Lemma 22 in \cite{kocak2020spectral}]
    Let $V$ be any real positive-definite matrix. For any vectors $\left\{\boldsymbol{x}_s\right\}_{1 \leq s<t}$ such that $\left\|\boldsymbol{x}_s\right\| \leq 1$ for all $1 \leq s<t$, we have that the determinant $ \left|V_t\right|$ of $V_t \triangleq V+\sum_{s=1}^{t-1} \boldsymbol{x}_s \boldsymbol{x}_s^{\top}$ is maximized when all $\boldsymbol{x}_s$ are aligned with an eigenvector of $V$.
    \label{lem10}
\end{lemma}

\begin{proof}{Proof of Lemma \ref{lem10}}
Let $V_t=V_{-i}+\boldsymbol{x}_i \boldsymbol{x}_i^{\top}$, that is $
V_{-i} = V+\sum_{\substack{s=1 \\ s \neq i}}^{t-1} \boldsymbol{x}_s \boldsymbol{x}_s^{\top}
$. Next, we will discuss two cases regarding the maximum value of $\left|V_t\right|$.

\textbf{Case 1: all eigenvalues have a multiplicity of 1.}

In this scenario, Lemma \ref{lem9} can be directly applied, which states that $\left|V_t\right|$ achieves its maximum value (maximizing only with respect to $\boldsymbol{x}_i$) when an eigenvector of $V_{-i}$ align with $\boldsymbol{x}_i$. Therefore, when $\left|V_t\right|$ attains its maximum value over all $\{\boldsymbol{x}_s\}_{s=1}^{t-1}$, each of these $\boldsymbol{x}_s$ aligns with a corresponding eigenvector of $V_{-s}$. Due to $V_t \boldsymbol{x}_i =V_{-i} \boldsymbol{x}_i+\boldsymbol{x}_i  \boldsymbol{x}_i ^{\top} \boldsymbol{x}_i=\sigma_{\min}(V_{-i})\boldsymbol{x}_i+\boldsymbol{x}_i$, these points of maximum values also align with the eigenvectors of $V_t$. Similarly, according to $V=V_t-\sum_{s=1}^{t-1} \boldsymbol{x}_s \boldsymbol{x}_s^{\top}$, we conclude that all $\boldsymbol{x}_s$ are aligned with the eigenvectors of $V$.

\textbf{Case 2: there exist eigenvalues with multiplicity greater than 1.}

In this case, in order to apply Lemma \ref{lem9}, we introduce a random perturbation with an amplitude not exceeding $c$. Subsequently, applying Lemma \ref{lem9} to the perturbed matrix, we obtain that the determinant $ \left|V_t^c\right|$ of $V_t^c \triangleq V^c+\sum_{s=1}^{t-1} \boldsymbol{x}_s^c \boldsymbol{x}_s^{c\top}$ is maximized when all $\boldsymbol{x}_s^c$ are aligned with an eigenvector of $V^c$. Furthermore, due to $\lim_{c \rightarrow 0} \left|V_t^c\right|=\left|V_t\right|$ and $\lim_{c \rightarrow 0} \boldsymbol{x}_s^c=\boldsymbol{x}_s$, therefore when all $\boldsymbol{x}_s$ are aligned with eigenvectors of $V$, we can obtain the maximum values of $|V_t|$.
\Halmos
\end{proof}

\begin{proof}{Proof of Lemma \ref{lem.E1}} 
We aim to bound the determinant ratio $\frac{\left|V_t\right|}{\left|V\right|}$ under the conditions $\{\left\|\boldsymbol{x}_s\right\| \leq 1, s=1,2,\cdots,t-1\}$. To achieve this objective, we first apply Lemma \ref{lem10} to obtain an upper bound on $\left|V_t\right|$. Let $V = UDU^\top$ be the full SVD, $\boldsymbol{e}_i$ is the basis vector of the identity matrix, then
\begin{equation}
  \begin{aligned}
     \left|V_t\right| &\leq \max _{\begin{array}{c} \boldsymbol{x}_1, \cdots, \boldsymbol{x}_{t-1} ; \\ \boldsymbol{x}_s \in\left\{U\boldsymbol{e}_1, \cdots, U\mathbf{e}_{d_1 d_2}\right\} \end{array} }\left|V+\sum_{s=1}^{t-1} \mathbf{x}_s \mathbf{x}_s^{\top}\right| =\max_{\begin{array}{c} \mathbf{x}_1, \cdots, \mathbf{x}_{t-1} ; \\ \mathbf{x}_s \in\left\{U\mathbf{e}_1, \cdots, U\mathbf{e}_{d_1 d_2}\right\} \end{array} }\left|UDU^\top+\sum_{s=1}^{t-1} U\mathbf{e}_s U\mathbf{e}_s^{\top}\right| \\
& =\max_{\begin{array}{c}\sum_{i=1}^{d_1 d_2} t_i=t-1 ;  \\ t_1, \cdots, t_{d_1 d_2} \text { positive integers}  \end{array}} \left|\operatorname{diag}\left(\sigma_i(V)+t_i\right)\right| \leq \max_{\begin{array}{c} \sum_{i=1}^{d_1 d_2} t_i=t-1; \\ t_1, \cdots, t_{d_1 d_2} \text {positive reals} \end{array} } \prod_{i=1}^{d_1 d_2}\left(\sigma_i(V)+t_i\right).
\end{aligned}  
\end{equation}
Furthermore, considering $|V|=\left|UDU^\top\right|=\prod_{i=1}^{d_1 d_2}\sigma_i(V)$, we have completed the proof.
\Halmos
\end{proof}

\section{Proof of the Theorem \ref{the5}}
\renewcommand{\thelemma}{F.\arabic{lemma}}  
\setcounter{lemma}{0}    

\begin{proof}{Proof of Theorem \ref{the5}}
We adjust the scale of the instantaneous regret in the $t$-th round according to the different stages of the algorithm. In the first stage, i.e., $t \in \{1,2,\cdots,T_1\}$, according to the Fundamental Theorem of Calculus and bounded norm assumption, we have $\left|\mu\left(\left\langle X^*, \Theta^*\right\rangle\right)-\mu\left(\left\langle X_t, \Theta^* \right\rangle\right)\right| \leq k_\mu \left|\left\langle X^*-X_t, \Theta^*\right\rangle\right|\leq k_\mu \left\|X^*-X_t\right\|_F \left\|\Theta^*\right\|_F\leq 2k_\mu$. In the second stage, we can derive the corresponding regret bound based on Theorem~\ref{the4}. 
We can further obtain the overall regret as follows:
\begin{equation}
  \begin{aligned}
    R_T&=\sum_{t=1 }^T \mu\left(\left\langle X^*, \Theta^*\right\rangle\right)-\mu\left(\left\langle X_t, \Theta^* \right\rangle\right)\\
&=\sum_{t=1}^{T_1} \left[\mu\left(\left\langle\mathcal{X}^*, \Theta^*\right\rangle\right)-\mu\left(\left\langle\mathcal{X}_t, \Theta^* \right\rangle\right)\right] +\sum_{t =1} ^{T-T_1}\left[\mu\left(\left\langle\mathcal{X}^*, \Theta^*\right\rangle\right)-\mu\left(\left\langle\mathcal{X}_{t}, \Theta^* \right\rangle\right)\right]\\
&\leq 2k_\mu T_1+\tilde{\mathcal{O}}\left( \omega k \sqrt{T} +  \tau T\right)\\
&= \tilde{\mathcal{O}} \left(\frac{  \zeta \sqrt{d_1 d_2 \gamma r T} }{c_r} \right).
\end{aligned}  
\end{equation}

\Halmos

\end{proof}

\section{GG-OFUL Algorithm}\label{appendixg}
In order to better elucidate the effectiveness of the algorithm proposed in this paper, we vectorize the matrix and deduce the corresponding UCB algorithm via graph Laplacian regularization term defined in Equation \eqref{eq1}, naming it GG-OFUL. This algorithm can also be regarded as a vectorized version of the algorithm \ref{algo1} presented in main text. In this case, this algorithm that corresponds to the estimator $\hat{\boldsymbol{\theta}}_t$ can be obtained by
\begin{equation}\label{eq44}
   \hat{\boldsymbol{\theta}}_t=\arg \min_{\boldsymbol{\theta} \in \mathbb{R}^{d_1d_2}} \sum_{i=1}^{T_1} \left[b (\langle \boldsymbol{x}_{h, i},\boldsymbol{\theta}\rangle)-y_{h, i} \langle \boldsymbol{x}_{h, i},\boldsymbol{\theta}\rangle \right]+ \sum_{i=1}^{t-1}\left[b (\langle \boldsymbol{x}_i,\boldsymbol{\theta}\rangle)-y_i \langle \boldsymbol{x}_i,\boldsymbol{\theta}\rangle \right]+\frac{1}{2}\|\boldsymbol{\theta}\|^2+\frac{\alpha}{2}\boldsymbol{\theta}^\top \widetilde{X}^\top L\widetilde{X} \boldsymbol{\theta}, 
\end{equation}
where $\boldsymbol{x}_i \in \mathcal{X}_{\text{vec }}\triangleq \left\{\text{vec}(X_i): X_i\in \mathcal{X}\right\}$. The $i$-th row of the matrix $\widetilde{X}$ corresponds to the $i$-th action, i.e., $\boldsymbol{x}_i$. 


Therefore, in the second stage, the main difference from GG-ESTT is that the identity matrix $I$ is replaced with a diagonal matrix $\Lambda$. As a result, the algorithm for GG-OFUL is as follows:
\begin{algorithm}[h]
\caption{GG-OFUL} 
\hspace*{0.02in} {\bf Input:} 
$\mathcal{X}_{\text {vec }}, T_1, T, L, \alpha, \delta$.
\begin{algorithmic}[1]
\State Let $\bar{V}(c_\mu)=\frac{I+a_\mu \alpha \widetilde{X}^\top L \widetilde{X}}{c_\mu}$;  $\bar{V}_1(c_\mu)=\bar{V}(c_\mu)+\sum_{i=1}^{T_1} \boldsymbol{x}_{h, i} \boldsymbol{x}_{h, i}^{\top}$, where the data $\left\{\boldsymbol{x}_{h, i}\right\}_{i=1}^{T_1}$ collected
in the first stage.
\For{$t=1$ to $T_1$} 
\State Choose action $\boldsymbol{x}_t \in \mathcal{X}_{\text {vec }}$ according to distribution $\mathbb{D}$, and receive reward $y_t$.
\EndFor
\For{$t=1$ to $T_2=T-T_1$} 
\State Compute $\hat{\boldsymbol{\theta}}_t$ according to Equation \eqref{eq44}.
\State Choose action $\boldsymbol{x}_t=\operatorname{argmax}_{\boldsymbol{x} \in \mathcal{X}_{\text {vec }}}  \left\{\mu \left(\langle \hat{\boldsymbol{\theta}}_t,\boldsymbol{x} \rangle \right)+\frac{k_\mu}{c_\mu} e_t \|\boldsymbol{x}\|_{\bar{V}_t^{-1}(c_\mu)} \right\}$ and receive reward $y_{t}$, where $e_t=\omega \sqrt{\log \frac{\left|\bar{V}_t(c_\mu)\right|}{|\bar{V}(c_\mu)| \delta^2}}+\sqrt{c_\mu}\left(1+\sqrt{\alpha \sigma_{\max}(\widetilde{X}^\top L \widetilde{X})}\right)$. The $i$-th row of the matrix $\widetilde{X}$ corresponds to the $i$-th action, i.e., $\boldsymbol{x}_i$.
\State Update $\bar{V}_{t+1}(c_\mu)=\bar{V}_t(c_\mu)+\boldsymbol{x}_t \boldsymbol{x}_t^{\top}$.
\EndFor
\end{algorithmic}
\label{algo3}
\end{algorithm}


\end{APPENDICES}


\end{document}